\documentclass{article}
\usepackage[final,nonatbib]{neurips_2020}

\usepackage{geometry}
\usepackage[utf8]{inputenc} 
\usepackage[T1]{fontenc}    
\usepackage{hyperref}       
\usepackage{url}            
\usepackage{booktabs}       
\usepackage{amsfonts}       
\usepackage{nicefrac}       
\usepackage{microtype}      
\usepackage{amsmath}

\usepackage{bbm}
\usepackage{dsfont}
\usepackage{color}

\usepackage[ruled,vlined,linesnumbered]{algorithm2e}
\usepackage{algcompatible}
\usepackage{bbm}
\usepackage{xcolor}
\usepackage{caption}
\usepackage{amssymb}
\usepackage{amsthm}

\newtheorem{theorem}{\protect\theoremname}
\newtheorem{definition}{\protect\definitionname}
\newtheorem{proposition}{\protect\propositionname}
\newtheorem{lemma}{\protect\lemmaname}
\newtheorem{Condition}{Condition}

\newtheorem{corollary}{Corollary}
\newtheorem{example}{Example}

\newcommand{\sigmab}{\underline{\sigma}}
\newcommand{\thetab}{\underline{\theta}}
\newcommand{\thetahat}{\widehat{\theta}}
\newcommand{\thetahatb}{\widehat{\thetab}}
\newcommand{\xb}{\underline{x}}
\newcommand{\yb}{\underline{y}}

\newcommand{\graph}{\mathcal{G}}
\newcommand{\T}{\mathcal{T}}
\newcommand{\R}{\mathcal{R}}
\newcommand{\K}{\mathcal{K}}
\newcommand{\V}{\mathcal{V}}
\newcommand{\Ed}{\mathcal{E}}
\newcommand{\chromo}{\chi}
\newcommand{\struct}{\mathbbm{S}}

\newcommand{\MaxFac}[1]{\mathcal{M}_{\rm{fac}}\left(#1\right)}
\newcommand{\Maxcli}[1]{\mathcal{M}_{\rm{cli}}\left(#1\right)}
\newcommand{\cspan}[1]{\left[#1\right]_{\rm{sp}}}
\newcommand{\E}[1]{{\mathbbm{E}}\left[ #1 \right]}
\newcommand{\var}[1]{{\rm{Var}}\left[ #1 \right]}
\newcommand{\Econd}[2]{{\mathbbm{E}}_{\left( #1 \right)}\left[ #2 \right]}
\newcommand{\varcond}[2]{{\rm{Var}}_{\left( #1 \right)} \left[ #2 \right]}

\newcommand{\powaset}{P}

\newcommand{\A}{\mathcal{A}}
\newcommand{\Sc}{\mathcal{S}}

\newcommand{\p}{\mathbbm{P}}
\newcommand{\Real}{\mathbbm{R}}

\newcommand{\gammahat}{\widehat{\gamma}}
\newcommand{\specialset}{parametrically complete}
\newcommand{\specialsetShort}{P.C.}

\providecommand{\definitionname}{Definition}
\providecommand{\lemmaname}{Lemma}
\providecommand{\propositionname}{Proposition}
\providecommand{\theoremname}{Theorem}

\author{
  Marc Vuffray, Sidhant Misra, Andrey Y. Lokhov \\
  Theoretical Division,\\
  Los Alamos National Laboratory, USA\\
  \texttt{\{vuffray, sidhant, lokhov\}@lanl.gov} 
}

\title{Efficient Learning of Discrete Graphical Models}

\begin{document}

\maketitle

\begin{abstract}
Graphical models are useful tools for describing structured high-dimensional probability distributions. Development of efficient algorithms for learning graphical models with least amount of data remains an active research topic. Reconstruction of graphical models that describe the statistics of \emph{discrete} variables is a particularly challenging problem, for which the maximum likelihood approach is intractable.
In this work, we provide the first sample-efficient method based on the \emph{Interaction Screening} framework that allows one to provably learn fully general discrete factor models with node-specific discrete alphabets and multi-body interactions, specified in an arbitrary basis.  We identify a single condition related to model parametrization that leads to rigorous guarantees on the recovery of model structure and parameters in any error norm, and is readily verifiable for a large class of models. Importantly, our bounds make explicit distinction between parameters that are proper to the model and priors used as an input to the algorithm. 
Finally, we show that the Interaction Screening framework includes all models previously considered in the literature as special cases, and for which our analysis shows a systematic improvement in sample complexity.
\end{abstract}


\global\long\def\argmin{\operatornamewithlimits{argmin}}
\global\long\def\argmax{\operatornamewithlimits{argmax}}

\section{Introduction}
Representing and understanding the structure of direct correlations between distinct random variables with graphical models is a fundamental task that is essential to scientific and engineering endeavors. It is the first step towards an understanding of interactions between interleaved constituents of elaborated systems \cite{Jansen2003}; it is key for developing causal theories \cite{chaves2015information}; and it is at the core of automated decision making \cite{constantinou2016complex}, cybersecurity \cite{buczak2016survey} and artificial intelligence \cite{wang2013markov}.

The problem of reconstruction of graphical models from samples traces back to the seminal work of \cite{ChowLiu1968} for tree-structured graphical models, and as of today is still at the center of attention of the learning community.
For factor models defined over general hypergraphs, the learning problem is particularly challenging in graphical models over discrete variables, for which the maximum likelihood estimator is in general computationally intractable.
One of the earlier tractable algorithms that has been suggested to provably reconstruct the structure of a subset of pairwise binary graphical models is based on inferring the sparsity pattern of the so-called regularized pseudo-likelihood estimator, equivalent to regularized logistic regression in the binary case \cite{RavikumarWainwrightLafferty2010}. However, additional assumptions required for this algorithm to succeed severely limit the set of pairwise binary models that can be learned \cite{MontanariPereira2009}. After it was proven that reconstruction of any discrete graphical models with bounded degree can be done in polynomial time in the system size \cite{BreslerMosselSly2013}, Bresler showed that it is possible to bring the computational complexity down to quasi-quadratic in the number of variables for Ising models (pairwise graphical models over binary variables); however, the resulting algorithm has non-optimal sample requirements that are double-exponential in other model parameters \cite{Bresler2015}. The first computationally efficient reconstruction algorithm for sparse pairwise binary graphical models with a near-optimal sample complexity with respect to the information theoretic lower bound \cite{SanthanamWainwright2012}, called \textsc{Rise}, was designed and analyzed in \cite{Vuffray2016nips}. The algorithm \textsc{Rise} suggested in this work is based on the minimization of a novel local convex loss function, called the \emph{Interaction Screening} objective, supplemented with an $\ell_1$ penalty to promote sparsity. Even though it has been later shown in \cite{lokhov2018science} that regularized pseudo-likelihood supplemented with a crucial post-processing step also leads to a structure estimator for pairwise binary models, strong numerical and theoretical evidence provided in that work demonstrated that \textsc{Rise} is superior in terms of worst-case sample complexity.

Algorithms for learning discrete graphical models beyond pairwise and binary alphabets have been proposed only recently in \cite{Ankur2017nips} and \cite{Klivans2017}. The method in \cite{Ankur2017nips} works for arbitrary models with bounded degrees, but being a generalization of Bresler's algorithm for Ising models \cite{Bresler2015}, it suffers from similar prohibitive sample requirements growing double-exponentially in the strength of model parameters. The so-called \textsc{Sparsitron} algorithm in \cite{Klivans2017} has the flavor of a stochastic first order method with multiplicative updates. It has a low computational complexity and is sample-efficient for structure recovery of two subclasses of discrete graphical models: multiwise graphical models over binary variables or pairwise models with general alphabets.
A recent follow-up work \cite{wu2018sparse} considered an $\ell_{2,1}$ constrained logistic regression, and showed that it provides a slight improvement of the sample complexity compared to \cite{Klivans2017} in the case of pairwise models over non-binary variables.

In this work, we propose a general framework for learning general discrete factor models expressed in an arbitrary parametric form. Our estimator termed \textsc{Grise} is based on a significant generalization of the Interaction Screening method of \cite{Vuffray2016nips, lokhov2018science}, previously introduced for pairwise binary models. Our primary insight lies in the identification of a single general condition related to model parameterization that is sufficient to obtain bounds on sample complexity. We show that this condition can be reduced to a set of local identifiability conditions that only depend on the size of the maximal clique of the factor graph and can be explicitly verified in an efficient way. We propose an iterative algorithm called \textsc{Suprise} which is based on \textsc{Grise} and show that it can efficiently perform structure and parameter estimation for arbitrary graphical models. Existing results  in the literature on this topic \cite{Vuffray2016nips,Ankur2017nips,Klivans2017,wu2018sparse} can be obtained as special cases of our general reconstruction results, which noticeably includes the challenging case of multi-body interactions defined over general discrete alphabets. Our theoretical guarantees can be expressed in any error norm, and explicitly includes distinction between bounds on the parameters of the underlying model and the prior parameters used in the optimization; as a result prior information that is not tight only has moderate effect on the sample complexity bounds. Finally, we also provide a fully parallelizable algorithmic formulation for the \textsc{Grise} estimator and \textsc{Suprise} algorithm, and show that they have efficient run times of $\widetilde{O}(p^{L})$ for a model of size $p$ with $L$-order interactions, that includes the best-known $\widetilde{O}(p^{2})$ scaling for pairwise models.

\section{Problem formulation}
In this Section, we formulate the general discrete graphical model selection problem that we consider and describe conditions that makes this problem well-posed.
\subsection{Parameterized family of models} \label{subsec:parameterized_family}
We consider positive joint probability distributions over $p$ variables $\sigma_{i} \in \A_{i}$ for $i = 1, \ldots, p$. The set of variable indices $i$ is referred to as vertices $\V=1, \ldots, p$. Node-dependent alphabets $\A_{i}$ are assumed to be discrete and of size bounded by $q>0$. Without loss of generality, the positive probability distribution over the $p$-dimensional vector $\sigmab$ can be expressed as
\begin{align}
    \mu(\sigmab) = \frac{1}{Z} \exp \left( \sum_{k \in \K} \theta^{*}_k f_k(\sigmab_k) \right),  \label{eq:mu_def}
\end{align}
where $\{f_k, \ k \in \K\}$ is a set of \emph{basis functions} acting upon subsets of variables $\sigmab_k\subseteq\sigmab$ that specify a family of distributions and $\theta^{*}_k$ are parameters that specify a model within this family. The quantity $Z$ denotes the partition function and serves as a normalization constant that enforces that the $\mu$ in \eqref{eq:mu_def} is a probability distribution.
For $i \in \{1,\ldots,p\}$, let $\K_i \subseteq \K$ denote the set of factors corresponding to basis functions acting upon subsets $\sigmab_{k}$ that contain the variable $\sigma_i$ and $|\K_i| = \mathbf{K}_i$.

Given any set of basis functions, we can locally center them by first defining for a given $i \in [p]$, the \emph{local centering functions}
\begin{align}
\phi_{ik}(\sigmab_{k \setminus i}) :=  \frac{1}{|\A_i|} \sum_{\sigma_i \in \A_i}f_k(\sigmab_{k}), \label{eq:centering}
\end{align}
where $\sigmab_{k \setminus i}$ denotes the vector $\sigmab_k$ without $\sigma_i$, and define the \emph{locally centered basis functions},
\begin{align}
    g_{ik}(\sigmab_k) =  f_{k}(\sigmab_{k}) - \phi_{ik}(\sigmab_{k \setminus i}). \label{eq:center_basis}
\end{align}
As their name suggests, the locally centered basis functions sum to zero $\sum_{\sigma_i \in \A_i} g_{ik}(\sigmab_k) = 0$.
To ensure the scales of the parameters are well defined, we assume that $\theta^{*}_{k}$ are chosen or rescaled such that all locally centered basis functions are \emph{normalized} in the following sense: 
\begin{align} 
\max_{\sigmab_k} |g_{ik}(\sigmab_k)| \leq 1,    \label{eq:normalization}
\end{align}
for all vertices $i\in\V$ and basis factor $k\in \mathcal{K}_i$. 
This normalization can always be achieved by choosing bounded basis functions $|f_k(\sigmab_k)|\leq 1/2$. An important special case is when the basis functions are already centered, i.e. $g_{ik}(\sigmab_k) = f_k(\sigmab_k)$. In this case the basis functions are directly normalized $\max_{\sigmab_k} |f_{k}(\sigmab_k)| = 1$. Note that one of the reasons to define the normalization in \eqref{eq:normalization} in terms of the centered functions $g_k$ instead of $f_k$ is to avoid spurious cases where the functions $f_k$ have inflated magnitudes due to addition of constants $f_k \leftarrow f_k + C$. In Appendix~\ref{app:well_poseness}, we show that the other important reason to employ centered functions is that degeneracy of the local parameterization with these functions translates to degeneracy of the global distribution in Eq.~\eqref{eq:mu_def}.

\subsection{Model selection problem} \label{subsec:model_selection}
For each $i \in [p]$, let $\T_i \subseteq  \K_i$ denote the set of \emph{target factors} that we aim at reconstructing accurately and let $\R_i = \K_i \setminus \T_i$ be the set of \emph{residual factors} for which we do not need learning guarantees. The target and residual parameters are defined similarly as $\thetab^{*}_{\T_{i}}=\{\theta_k^{*} \mid k\in \T_i\}$ and $\thetab^{*}_{\R_{i}}=\{\theta_k^{*} \mid k\in \R_i\}$ respectively. Given independent samples from a model in the family in Section~\ref{subsec:parameterized_family}, the goal of the model selection problem is to reconstruct the target parameters of the model.

\begin{definition}[Model Selection Problem]  \label{def:model_selection_problem}
Given $n$ i.i.d. samples $\sigmab^{(1)}, \ldots, \sigmab^{(n)}$ drawn from some distribution $\mu(\sigmab)$ in Eq.~\eqref{eq:mu_def} defined by $\thetab^*$, and prior information on $\thetab^*$ given in form of an upper bound on the $\ell_1$-norm of the local sub-components
\begin{align}   \label{eq:gammahat}
    \|\thetab^*_i\|_1 \le \gammahat,
\end{align}
and a \emph{local constraint set} $\mathcal{Y}_i \subseteq \Real^{\mathbf{K}_i}$ for each $i \in [p]$ such that 
\begin{align}
    \thetab_i^* \in \mathcal{Y}_i,  \label{eq:local_constraint}
\end{align}
compute estimates $\widehat\thetab$ of $\thetab^*$ such that the estimates of the target parameters satisfy
\begin{align}
    \|\widehat\thetab_{\T_i}-\thetab^*_{\T_i}\| \leq \frac{\alpha}{2}, \quad \forall i \in [p], \label{eq:requirement}
\end{align}
where $\|\cdot\|$ denotes some norm of interest with respect to which the error is measured.
\end{definition}
The bound on the $\ell_1$-norm in \eqref{eq:gammahat} is a natural generalization of the sparse case where $\thetab^*$ only has a small number of non-zero components; in the context of parameter estimation in graphical models, the setting of parameters bounded in the $\ell_1$-norm has been previously considered in \cite{Klivans2017}. The constraint sets $\mathcal{Y}_i$ are used to encode any other side information that may be known about the model. 

\subsection{Sufficient conditions for well-posedness}  \label{subsec:conditions}
We describe some conditions on the model in \eqref{eq:mu_def} that makes the model selection problem in Definition~\ref{def:model_selection_problem} well-posed. We first state the conditions formally.
\begin{Condition}   \label{master_condition}
    The model from which the samples are drawn in the model selection problem in Definition~\ref{def:model_selection_problem} satisfies the following:
    \begin{itemize} 
    \item[(C1)]  \textbf{Local Learnability Condition for Graphical Models:}  
    There exists a constant $\rho_i>0$ such that for every vertex $i$ and any vector in the perturbation set $\xb \in \mathcal{X}_i \subseteq \Real^{\mathbf{K}_i}$ defined as
    \begin{align}   \label{cond:diff_set}
        \mathcal{X}_{i} = \{\xb  = \yb_1-\yb_2 \mid \yb_1,\yb_2  \in \mathcal{Y}_i, \|\yb_1\|_1 \leq \gammahat, \|\yb_2\|_1 \leq \gammahat\},
    \end{align}
     the following holds:
        \begin{align}
            \E{\left(\sum_{k \in \K_i} x_k g_{ik} (\sigmab_k) \right )^2} \geq \rho_i \|\xb_{\T_i}\|^2,\label{eq:c1_fisher}
        \end{align}
         where $\xb_{\T_i}$ denotes the  components $k\in \T_i$ of $x$, and $\|\cdot\|$ is the norm used in Definition~\ref{def:model_selection_problem}.
    \item[(C2)]    \textbf{Finite Maximum Interaction Strength:} The following quantity $\gamma$ is finite,
        \begin{align}
            \gamma = \max_{i\in \V} | \max_{\sigmab}\sum_{k \in \K_i} \theta^{*}_k g_{ik}(\sigmab_k)|<\infty. \label{eq:gammadef}
        \end{align}
    \end{itemize}
\end{Condition}
Condition (C1) consists in satisfying the inequality in Eq.~\eqref{eq:c1_fisher} involving a quadratic form $\xb^{\top} \widetilde{I} \underline{x}$ where the matrix $\widetilde{I}$ has indices $k,k' \in \K_i$ and is explicitly defined as $\widetilde{I}_{k,k'} =  \E{g_{ik}(\sigmab_k) g_{ik'}(\sigmab_{k'})}$.
This matrix $\widetilde{I}$ is in fact related to the \emph{conditional} Fisher information matrix.

The conditional Fisher information matrix $I$ with indices $k,k' \in \K_i$ is derived from the conditional distribution of $\sigma_i$ given the remaining variables and reads,
\begin{align}
            I_{k,k'} 
            &= \E{g_{ik}(\sigmab_k) g_{ik'}(\sigmab_{k'})} - \Econd{\sigmab_{\setminus i}}{\Econd{\sigma_i \mid \sigmab_{\setminus i}}{g_{ik}(\sigmab_k)}\Econd{\sigma_i \mid \sigmab_{\setminus i}}{g_{ik'}(\sigmab_{k'})}}.
\end{align}
We immediately see that the matrix $\widetilde{I}$ dominates the conditional Fisher information matrix in the positive semi-definite sense, that is $\xb^{\top} \widetilde{I}(\theta^*) \xb \geq \xb^{\top} I(\theta^*) \xb$ for all $\xb \in \Real^{\mathbf{K}_i}$. Therefore, Condition \emph{(C1)} is satisfied whenever the conditional Fisher information matrix is non-singular in the parameter subspace $\xb_{\T_i}$ that we care to reconstruct and which is compatible with our priors, i.e. for $\xb \in \mathcal{X}_{i}$. We would like to add that the conditional Fisher information matrix is a natural quantity to consider in this problem as we deliberately focus on using conditional statistics rather than global ones in order to bypass the intractability of the global log-likelihood approach. We are strongly convinced that it should appear in any analysis that entails conditional statistics.

Condition \emph{(C2)} is required to ensure that the model can be recovered with finitely many samples. For many special cases, such as the Ising model, the minimum number of samples required to estimate the parameters must grow exponentially with the maximum interaction strength \cite{SanthanamWainwright2012}. A more detailed discussion about well-posedness and Conditions \emph{(C1)} and \emph{(C2)} can be found in Appendix~\ref{app:well_poseness}.

Conditions \emph{(C1)} and \emph{(C2)} differ from the concepts in \cite{Vuffray2016nips} called restricted strong convexity property and bound on the interaction strength, respectively, in a subtle but critical manner. Conditions \emph{(C1)} can be identified with restricted strong convexity only when the $\ell_2$-norm is used in Eq.~\eqref{eq:c1_fisher}. We will see later that the notion of restricted strong convexity is not required for the $\ell_\infty$-norm that appears to be a natural metric for which the local learnability condition can be verified for general models. Moreover, for general models it remains unclear whether the restricted strong convexity holds for values of $\rho_i$ that are independent of the problem dimension $p$. Condition \emph{(C2)} is a weaker assumption than the bound on the interaction strength from [17] for it does not require an extra assumption on the maximum degree of the graphical model.

\section{Generalized interaction screening} 
In this Section, we introduce the algorithm that efficiently solves the model selection problem in Definition~\ref{def:model_selection_problem} and provides rigorous guarantees on its reconstruction error and computational complexity. 

\subsection{Generalized regularized interaction screening estimator}
We propose a generalization of the estimator \textsc{Rise}, first introduced in \cite{Vuffray2016nips} for pairwise binary graphical models, in order to reconstruct general discrete graphical models defined in \eqref{eq:mu_def}. The \emph{generalized interaction screening objective} (GISO) is defined for each vertex $u$ separately and is given by 
\begin{align}
    \Sc_n(\thetab_u) = \frac{1}{n} \sum_{t=1}^{n} \exp \left( -\sum_{k \in \K_u} \theta_k g_{uk}(\sigmab_k^{(t)})  \right),
\end{align}
where $\sigmab^{(1)}, \ldots, \sigmab^{(n)}$ are $n$ i.i.d samples drawn from $\mu(\sigmab)$ in Eq.~\eqref{eq:mu_def}, $\thetab_u := (\thetab_k)_{k \in \K_u}$ is the vector of parameters associated with the factors in $\K_u$ and the locally centered basis functions $g_{uk}$ are defined as in Eq.~\eqref{eq:center_basis}.
The GISO retains the main feature of the interaction screening objective (ISO) in \cite{Vuffray2016nips}: it is proportional to the inverse of the factor in $\mu(\sigmab)$, except for the additional centering terms $\phi_{uk}$. The GISO is a convex function of $\thetab_u$ and retains the ``screening'' property of the original ISO. The GISO is used to define the generalized regularized interaction screening estimator (\textsc{Grise}) for the parameters given by 
\begin{align}   \label{eq:GRISE}
    \widehat{\thetab}_u = \argmin_{\thetab_u\in \mathcal{Y}_u:\|\thetab_{u}\|_{1} \leq \gammahat} \Sc_n(\thetab_u),
\end{align}
where $\gammahat$ and $\mathcal{Y}_u$ are the prior information available on $\thetab_u^*$ as defined in \eqref{eq:gammahat} and \eqref{eq:local_constraint}.

\subsection{Error bound on parameter estimation with \textsc{Grise}}
We now state our first main result regarding the theoretical guarantees on the parameters reconstructed by \textsc{Grise}. 
We call $\widehat\thetab_u$ an $\epsilon$-optimal solution of \eqref{eq:GRISE} if
\begin{align}\label{eq:epsilon_solution}
\Sc_n(\widehat{\thetab}_u) \leq \min_{\thetab_u\in \mathcal{Y}_u:\|\thetab_{u}\|_{1} \leq \gammahat} \Sc_n(\thetab_u) + \epsilon.
\end{align}
\begin{theorem}[Error Bound on \textsc{Grise} Estimates] \label{thm:grise}
Let $\sigmab^{(1)}, \ldots, \sigmab^{(n)}$ be i.i.d. samples drawn according to $\mu(\sigmab)$ in \eqref{eq:mu_def}. For some node $u\in\V$, assume that the model satisfies Condition~\ref{master_condition} for some norm $\|\cdot\|$ and some constraint set $\mathcal{Y}_u$, and let $\alpha>0$ be the prescribed accuracy level.
If the number of samples satisfies
\begin{align}
     n\geq 2^{14} \frac{\gammahat^2 (1+\gammahat)^2 e^{4\gamma}}{\alpha^{4} \rho_u^{2}} \log(\frac{4\mathbf{K}_u^2}{\delta}),
\end{align}
then, with probability at least $1 - \delta$, any estimate that is an $\epsilon$-minimizer of \textsc{Grise}, with $ \epsilon \leq (\rho_u \alpha^{2} e^{-\gamma})/(20(1+\gammahat))$, satisfies $    \|\widehat{\thetab}_{\T_u} - \thetab^{*}_{\T_u} \| \leq  \frac{\alpha}{2}$.
\end{theorem}
The proof of Theorem~\ref{thm:grise} can be found in Appendix~\ref{app:proof_theorems}.

The computational complexity of finding an $\epsilon$-optimal solution of \textsc{Grise} for a trivial constraint set $\mathcal{Y}_u = \Real^{\mathbf{K}_u}$ is $C \frac{c_g n \mathbf{K}_u}{\epsilon^2}  \ln (1 + \mathbf{K}_u)$, where $c_g$ is an upper-bound on the computational complexity of evaluating any $g_{ik}(\sigmab_k)$ for $k \in \K_i$, and  $C$ is a universal constant independent of all the parameters of the problem, see Proposition~\ref{prop:computational_complexity} in Appendix~\ref{app:computational_complexity}.
For a certain class of constraint sets $\mathcal{Y}_u$, which we term \specialset , the problem can be solved in two steps: first, finding a solution to an unconstrained problem, and then projecting onto this set. Note, however, that in general the problem of finding $\epsilon$-optimal solutions to constrained \textsc{Grise} can still be difficult since the constraint set $\mathcal{Y}_u$ can be arbitrarily complicated.
\begin{definition}  \label{def:special_constraint_set}
The constraint set $\mathcal{Y}_u$ is called a \specialset\ set if for all $\thetab_u \in \Real^{|\mathbf{K}_u|}$, there exists $\widehat\thetab_u \in \mathcal{Y}_u$ such that for all $\sigmab_u$, we have
\begin{align}
    \sum_{k \in \K_u} \thetab_k g_{uk}(\sigmab_k) = \sum_{k \in \K_u} \widehat\thetab_k g_{uk}(\sigmab_k).  \label{eq:degeneracy}
\end{align}
Any $\widehat\thetab_k \in \mathcal{Y}_u$ satisfying \eqref{eq:degeneracy} is called an \emph{equi-cost projection} of $\thetab_u$ onto $\mathcal{Y}_u$ and is denoted by
\begin{align}
    \widehat\thetab_u \in \mathcal{P}_{\mathcal{Y}_u}(\thetab_u).   \label{eq:projection}
\end{align}
\end{definition}

The computational complexity of finding of an $\epsilon$-optimal solution of \textsc{GRISE} with parametrically complete set is $C \frac{c_g n \mathbf{K}_u }{\epsilon^2} \ln (1 + \mathbf{K}_u) + \mathcal{C}(\mathcal{P}_{\mathcal{Y}_u}(\widehat\thetab_u^{\text{unc}}))$, where $\mathcal{C}(\mathcal{P}_{\mathcal{Y}_u}(\widehat\thetab_u^{\text{unc}}))$ denotes the computational complexity of the projection step, see Theorem~\ref{th:computational_complexity} in Appendix~\ref{app:computational_complexity}. 
 
As we will see, for many graphical models it is often possible to explicitly construct \specialset\ sets for which the computational complexity of the projection step $\mathcal{C}(\mathcal{P}_{\mathcal{Y}_u}(\widehat\thetab_u^{\text{unc}}))$ is insignificant compared to the computational complexity of unconstrained \textsc{Grise}.

\section{Structure identification and parameter estimation}
In this Section we show that the structure of graphical models, which is the collection of maximal subsets of variables that are associated through basis functions, as well as the associated parameters, can be efficiently recovered. The key elements are twofold. First, we prove that for maximal cliques, the Local Learnability Condition (LLC) in \emph{(C1)} can be easily verified and yields a LLC constant independent of the system size. Second, we leverage this property to design an efficient structure and parameter learning algorithm coined \textsc{Suprise} that requires iterative calls of \textsc{Grise}.

\subsection{The structure of graphical models}\label{sub:structure_GM}
The structure plays a central role in graphical model learning for it contains all the information about the conditional independence or Markov property of the distribution $\mu(\sigmab)$ from Eq~\eqref{eq:mu_def}. In order to reach the definition of the structure presented in Eq.~\eqref{eq:def_structure}, we have to introduce graph theoretic concepts specific to graphical models.

The factor graph associated with the model \emph{family} is a bipartite graph $\graph = \left(\V,\K,\Ed\right)$ with vertex set $\V$, factor set $\K$ and edges connecting factors and vertices,
\begin{align}
 \Ed = \left\{(i,k) \subseteq \V \times \K \mid \sigma_i \in \sigmab_k \right\}.\label{eq:factor_graph_edges}
\end{align}
We see from Eq.~\eqref{eq:factor_graph_edges} that the edge $(i,k)$ exists when the variable $\sigma_i$ associated with the vertex $i$ is an argument of the basis function $f_k(\sigmab_k)$ associated with the factor $k$. Note that this definition only depends on the set of basis functions $\K$ and does not refer to a particular choice of model within the family. The factor graph $\graph^* = \left(\V,\K^*,\Ed^*\right)$ associated with a \emph{model}, as defined in Eq.~\eqref{eq:mu_def}, is the induced subgraph of $\graph$ obtained from the vertex set $\V$ and factor subset $\K^* = \left\{k\in \K \mid  \theta_k^* \neq 0 \right\}$. We also use the shorthand notation $\graph^* = \graph\left[(\V,\K^*)\right]$ to denote an induced subgraph of $\graph$.

We define the neighbors of a factor $k$ as the set of vertices linked by an edge to $k$ and denote it by $\partial k = \left\{i \in \V \mid (i,k) \in \Ed \right\}$.
The largest factor neighborhood size $L = \max_{k\in \K} \left| \partial k\right|$
is called the interaction order.
Families of graphical models with $L\leq2$ are referred to as pairwise models as opposed to the generic denomination of $L$-wise models when $L$ is expected to be arbitrary.

The set of maximal factors of a graph is the set of factors whose neighborhood is not strictly contained in the neighborhood of another factor,
\begin{align}
     \MaxFac{\graph} = \left\{ k \in \K \mid \nexists k' \in \K \ \text{s.t} \ \partial k \subset \partial k' \right\}.
\end{align}
Notice that multiple maximal factors may have the same neighborhood. This motivates the definition of the set of maximal cliques which is contained in the powerset $\powaset (\V)$ and consists of all neighborhoods of maximal factors,
\begin{align}
     \Maxcli{\graph} = \left\{ c \in \powaset (\V) \mid \exists k \in \MaxFac{\graph}  \text{ s.t. }  c=\partial k\right\}.
\end{align}
The set of factors whose neighborhoods are the same maximal clique $c\in \Maxcli{\graph}$ is called the span of the clique defined as $\cspan{c} = \left\{ k \in \MaxFac{\graph} \mid c = \partial k \right\}$.
Finally, the structure $\struct$ of a graphical model is the set of maximal cliques associated with the factor graph of the \emph{model},
\begin{align}
    \struct(\graph^*) = \Maxcli{\mathcal{\graph^*}}.\label{eq:def_structure}
\end{align}
We would like to stress that the structure of a model is different from the set of maximal cliques of the \emph{family} of graphical models $\Maxcli{\graph}$ as the former is constructed with the set of factors associated with non-zero parameters while the latter consists of all potential maximal factors.

\subsection{From local learnability condition to nonsingular parametrization of cliques}\label{sub:NPC}
We show that the learning problem of reconstructing maximal cliques is well-posed in general and especially for non-degenerate globally centered basis functions. To this end, we demonstrate that the LLC in \emph{(C1)} is automatically satisfied whenever the target sets $\T_i$ consist of factors corresponding to maximal cliques of the graphical model family. Importantly, we prove that the LLC constant $\rho_i$ does not depend on the dimension of the model for the $\ell_{\infty,2}$-norm but rather relies on the Nonsingular Parametrization of Clique  (NPC) by the basis functions. Similarly, we also guarantee that the LLC holds for the $\ell_{2}$-norm in the case of pairwise colorable models.

We introduce globally centered basis functions defined for any factor $k \in \K$ through the inclusion–exclusion formula,
\begin{align}
    h_k\left(\sigmab_k\right) = f_{k}(\sigmab_k) + \sum_{r\in \powaset(\partial k)\setminus \emptyset} \frac{(-1)^{|r|}}{|\A_r|}\sum_{\sigmab_r}f_{k}(\sigmab_k),\label{eq:globally_centered_factors}
\end{align}
where $\A_r = \bigotimes_{j\in r}\A_j$. It is straightforward to see that globally centered functions sum partially to zero for any variables, i.e. $\sum_{\sigma_i \in \A_i} h_{k}(\sigmab_k)=0$ for all $i\in \partial k$. It is worth noticing that when the functions $f_k$ are already globally centered, we have $f = g = h$. We would also like to point out that unlike locally centered functions $g_{ik}$, globally centered functions cannot in general be interchanged with functions $f_k$ without modifying conditional distributions. However they play an important role in determining the independence of basis functions around cliques through the Nonsingular Parametrization of Cliques (NPC) constant introduced below. Given a perturbation set $\mathcal{X}_i$, as defined in Eq.~\eqref{cond:diff_set}, the NPC constant is defined through the following minimization, 
\begin{align}   
  \rho^{\rm{NPC}}_i = \min_{\substack{c \in \Maxcli{\graph} \\ c \ni i}} \min_{\substack{\|\xb_c\|_2 = 1 \\ \xb_c \in \mathcal{X}^c_i}} \Econd{\sigma_i}{\sum_{\sigmab_{c\setminus i}\in \A_{c\setminus i}}\left(\sum_{k\in \cspan{c}}x_k h_{k}(\sigmab_k)\right)^2},\label{eq:npc_const}
\end{align}
where the vector $\xb_c \in \Real^{|\cspan{c}|}$ belongs to $\mathcal{X}^c_i$, the projection of the constraint $\mathcal{X}_i \subseteq \Real^{\mathbf{K}_i}$ to the components $k \in \cspan{c}$ and the expectation is with respect to the marginal distribution of $\sigma_i$. Note that NPC constant only depends on $L$ and not on the size of the system, and can be explicitly computed in time $O(\mathbf{K})$. A detailed discussion can be found in Appendix~\ref{app:structure}.
The importance of the NPC constant is highlighted by the following proposition that guarantees that the LLC is satisfied for maximal factors in $\ell_{\infty,2}$-norm as long as $\rho^{\rm{NPC}}_i>0$. 

\begin{proposition}[LLC in $\ell_{\infty,2}$-norm]\label{prop:L_wise_NPC}
For a specific vertex $i \in \V$, let the target set be maximal factors with $i$ as neighbor, i.e. $\T_i = \left\{k \in \MaxFac{\graph} \mid \partial k \ni i\right\}$. For vectors in the perturbation set $\xb \in \mathcal{X}_i \subseteq \Real^{\mathbf{K}_i}$,  define the $\ell_{\infty,2}$-norm over components that are maximal factors as $\|\xb_{\T_i}\|_{\infty,2} = \max_{\substack{ c \in \Maxcli{\graph}\\ c \ni i }} \sqrt{\sum_{k \in \cspan{c}} x^2_k}$.
Then for discrete graphical models with maximum alphabet size $q$, interaction order $L$ and models with finite maximum interaction strength $\gamma$ as defined in Eq.~\eqref{eq:gammadef}, the Local Learnability Condition \emph{(C1)} is satisfied whenever the Nonsingular Parameterization of Cliques constant $\rho^{\rm{NPC}}_i$ is nonzero and we have,
\begin{align}
    \E{ \left(\sum_{k \in \K_i} x_k g_{ik} (\sigmab_k) \right )^2} \geq  \rho^{\rm{NPC}}_i \left( \frac{\exp(-2 \gamma)}{q}\right)^{L-1} \|\xb_{\T_i}\|_{\infty,2}^2.
\end{align}
\end{proposition}
Proposition~\ref{prop:L_wise_NPC} guarantees that the LLC can be satisfied uniformly in the size $p$ of the model whenever $\rho_i^{\rm{NPC}}>0$. The proof of Proposition~\ref{prop:L_wise_NPC} can be found in Appendix~\ref{app:structure}.

For family of models whose factors involve at most $L=2$ variables, the so-called pairwise models, we can show that the LLC conditions for maximal factors also holds for the $\ell_{2}$-norm. This LLC conditions depends on the vertex chromatic number $\chromo$ of the \emph{model} factor graph. We recall that a vertex coloring of a graph $\graph^* = \left(\V,\K^*,\Ed^*\right)$ is a partition $\{S_r\}_{r\in \mathbb{N}} \in \powaset(\V)$ of the vertex set such that no two vertices with the same color are connected to the same factor node, i.e. $i,j \in S_r \Rightarrow \nexists k\in \K^*$ s.t. $i,j\in \partial k$. The chromatic number is the cardinality of the smallest graph coloring.

\begin{proposition}[LLC in $\ell_{2}$-norm for pairwise models]\label{prop:pairwise_NPC}
For a specific vertex $i \in \V$, let the target set be maximal factors with $i$ as neighbor, i.e. $\T_i = \left\{k \in \MaxFac{\graph} \mid \partial k \ni i\right\}$. For $\xb \in \mathcal{X}_i \subseteq \Real^{\mathbf{K}_i}$,  define the $\ell_{2}$-norm over components that are maximal factors $ \|\xb_{\T_i}\|_{2} = \sqrt{\sum_{k \in \T_i} x^2_k}$.
Then for discrete pairwise graphical models with maximum alphabet size $q$ and models with chromatic number $\chromo$ and finite maximum interaction strength $\gamma$ as defined in Eq.~\eqref{eq:gammadef}, the Local Learnability Condition \emph{(C1)} is satisfied whenever the NPC constant $\rho^{\rm{NPC}}_i$ is nonzero and we have,
\begin{align}
    \E{ \left(\sum_{k \in \K_i} x_k g_{ik} (\sigmab_k) \right )^2} \geq  \frac{\rho^{\rm{NPC}}_i}{\chromo}  \frac{\exp(-2 \gamma)}{q} \|\xb_{\T_i}\|_{2}^2.
\end{align}
\end{proposition}
The reader will find the proof of Proposition~\ref{prop:pairwise_NPC} in Appendix~\ref{app:structure}.

\subsection{Structure unveiling and parameter reconstruction with interaction screening estimation}\label{subsec:suprise}
\begin{algorithm}
\SetAlgoLined
\tcp{Step 1:~Initialization of set of considered factors}
$\K^0 \leftarrow \K$ \;

\For{$t = 0,\ldots, L-1$}{
    \tcp{Step 2:~Reconstruct maximal factors bigger than $L-t$} 
     Construct the induced sub-graph: $\graph^t \leftarrow \graph\left[(\V,\K^t)\right]$\;
    \For{$u \in \V$}{
    Set target factors: $\T_u^{t} \leftarrow \left\{k \in \MaxFac{\graph^t} \mid \partial k \ni u\right\}$\;
    Set residual factors: $\R_u^{t} \leftarrow \K_u^{t} \setminus \T_u^{t}$\;
      Estimate $\thetahatb_{u}^t$ using \textsc{Grise} with accuracy at least $\epsilon = \rho_{\rm{NPC}} \alpha^2 \exp(-\gamma (2L-1)) / (20(1+\gammahat) q^{L-1})$ on the model defined by $\K_u^t, \T_u^t, \R_u^t$ and constraint set $\mathcal{Y}_u$\;
    }
     \tcp{Step 3:~Identify max cliques associated with zero parameters}
        Initialize set of removable factors: $\mathcal{N}^t \leftarrow \emptyset$\;
         \For{$c\in \Maxcli{\graph^t}$}{
         Compute average reconstruction: $\thetahatb^{{\rm{avg}}(t)}_c \leftarrow \left\{\lvert c\rvert^{-1} \sum_{u \in c} (\thetahatb^t_{u})_k  \mid k \in \cspan{c} \right\}$\;
         \If{$\| \thetahatb^{{\rm{avg}}(t)}_c \|_2 < \alpha /2$}{
          Update set of removable factors: $\mathcal{N}^t \leftarrow \mathcal{N}^t \cup \cspan{c}$ \;
            }
         }  
    Update considered factors: $\K^{t+1} \leftarrow \K^{t} \setminus \mathcal{N}^t$\;
}
    \tcp{Step 4:~output structure and non-zero parameters of maximal factors} 
{\bf return} $\widehat{\struct} = \Maxcli{\graph\left[(\V,\K^L)\right]}$ and $\thetahatb_{\mathcal{M}} = \left\{ \thetahat^{{\rm{avg}}(L-1)}_k \mid k \in \MaxFac{\graph\left[(\V,\K^L)\right]} \right\}$\;

\caption{Structure Unveiling and Parameter Reconstruction with Interaction Screening Estimation (\textsc{Suprise})}
\label{alg:suprise}
\end{algorithm}

Suppose that we know $\alpha>0$, a lower-bound on the minimal intensity of the parameters associated with the structure in the sense that $\alpha \leq \min_{c\in\struct(\graph^*)} \sqrt{\sum_{k\in \cspan{c}} {\theta^*_k}^2}$.
Then we can recover the structure and parameters associated with maximal factors of any graphical models using Algorithm~\ref{alg:suprise}, coined \textsc{Suprise} for Structure Unveiling and Parameter Reconstruction with Interaction Screening Estimation. \textsc{Suprise} that implements an iterative use of \textsc{Grise} is shown to have a sample complexity logarithmic in the system size for models with non-zero NPC constants. Our second main result is the following Theorem~\ref{thm:structure_general}, proved in Appendix~\ref{app:structure}, which provides guarantees on \textsc{Suprise}.

\begin{theorem}[\textbf{Reconstruction and Estimation Guarantees for \textsc{Suprise}}]\label{thm:structure_general}
Let $\mu(\sigmab)$ in \eqref{eq:mu_def} be the probability distribution of a discrete graphical model with maximum alphabet size $q$, interaction order $L$, finite maximum interaction strength $\gamma$ and smallest Nonsingular Parameterization of Cliques constant greater than zero, i.e. $\rho_{\rm{NPC}} = \min_{u\in \V} \rho^{\rm{NPC}}_u>0$. Let $\sigmab^{(1)}, \ldots, \sigmab^{(n)}$ be i.i.d. samples drawn according to $\mu(\sigmab)$ and assume that
\begin{align}
   n\geq 2^{14} q^{2\left(L-1\right)} \frac{\gammahat^2 (1+\gammahat)^2 e^{4\gamma L}}{\alpha^{4} \rho_{\rm{NPC}}^2} \log\left(\frac{4 p L \mathbf{K}^2}{\delta}\right),   \label{eq:n_suprise}
\end{align}
where $\mathbf{K} = \max_{u\in \V} \mathbf{K}_u$ is the maximal number of basis functions in which a variable can appear and $\gammahat \geq \max_{u \in \V} \|\thetab^*_u\|_1 $ is our $\ell_1$-prior on the components of the parameters. Then the structure of the general graphical model is perfectly recovered using Algorithm~\ref{alg:suprise}, i.e. $\widehat{\struct} = \struct$, with probability $1-\delta$. In addition, the parameters associated with maximal factors are reconstructed with precision $\max_{ c \in \struct} \sum_{k \in \cspan{c}} \left(\thetahat_k - \theta^*_k \right)^2 \leq \alpha^2/4$ for general models and with $\sum_{ c \in \struct} \sum_{k \in \cspan{c}} \left(\thetahat_k - \theta^*_k \right)^2 \leq \chromo^2 \alpha^2 / 4$ for pairwise models with chromatic number $\chromo$.

The total computational complexity scales as $\widetilde{\mathcal{O}}(p \mathbf{K})$, for fixed $L$, $\alpha$, $\gamma$, $\gammahat$ and $\delta$, if the constraint sets $\mathcal{Y}_u$ are \specialset. 
\end{theorem}

As an application of Theorem~\ref{thm:structure_general}, we show the sample and computational complexity of recovering parameter values and the structure of some well-known special cases in Table~\ref{table:sc_performances}. 
The parameter $\alpha$ appearing in Table~\ref{table:sc_performances} is the precision to which parameters are recovered in the considered norm, $\chi$ is the chromatic number of the graph, $L$ is the interaction order, $q$ is the maximum alphabet size, $\gamma$ is the maximum interaction strength and $p$ is the number of variables. At this point, it is instructive to compare our sample complexity requirements to existing results. A direct application of bounds of \cite{Klivans2017} and \cite{wu2018sparse} to the case of pairwise multi-alphabet models that we consider below yields $O(\exp(14\gamma))$ dependence, whereas \textsc{Suprise} has a complexity that scales as $O(\exp(12\gamma))$. In the case of binary $L$-wise models, while \cite{Klivans2017} shows the $O(\exp(O(\gamma L)))$ scaling, \textsc{Suprise} enjoys a sample complexity $O(\exp(4\gamma L))$. The algorithm of \cite{Ankur2017nips} recovers a subclass of general graphical models with bounded degree, but has a sub-optimal double-exponential scaling in $\gamma$, while \textsc{Suprise} leads to recovery of arbitrary discrete graphical models with a single-exponential dependence in $\gamma$ and needs no bound on the degree. In terms of the computational complexity, \textsc{Suprise} achieves the efficient scaling $\widetilde{O}(p^{L})$ for models with the maximum interaction order $L$, which matches the best-known $\widetilde{O}(p^{2})$ scaling for pairwise models \cite{Klivans2017, wu2018sparse}. In summary, \textsc{Suprise} generalizes, improves and extends the existing results in the literature. The proofs for special cases can be found in Section~\ref{sec:special_cases} of the Supplementary Material.

\begin{table}
\caption{Sample complexity and computational complexity of \textsc{Suprise} over special cases.}
\small
\centering
\setlength{\tabcolsep}{2pt}
\def\arraystretch{1.4}
\begin{tabular}{ |c|c|c|c|c|c| } 
\hline 
Model name & Inter. order & Alphabet size & Recovery type& Sample complexity& Algo. complexity  \\

\hline
Ising & 2 & 2 & structure & $O\left(\alpha^{-4} e^{8 \gamma} \log p  \right)$ & $\widetilde{O}(p^{2}) $ \\
\hline
Ising & 2 & 2 & $\ell_2$-parameter & $O\left(\chromo^2 \alpha^{-4}  e^{8 \gamma} \log p \right)$ & $\widetilde{O}(p^{2})$ \\ 
\hline
Binary & $L$ & 2 & structure & $O\left(\alpha^{-4} 4^L e^{4\gamma L} L \log p\right)$ & $\widetilde{O}(p^{L})$ \\
\hline
Pairwise & 2 & $q$ & structure & $O\left(\alpha^{-4} q^4 e^{12\gamma } \log (p q)\right)$ & $\widetilde{O}(p^{2})$ \\
\hline
Pairwise & 2 & $q$ & $\ell_2$-parameter & $O\left(\chromo^2 \alpha^{-4}  q^4 e^{12\gamma } \log (p q)\right)$ & $\widetilde{O}(p^{2})$ \\
\hline
General & $L$ & $q$ & structure & $O\left(\alpha^{-4} q^{2L} e^{4\gamma (L+1)} L \log (p q)\right)$ & $\widetilde{O}(p^{L})$ \\
\hline
\end{tabular}

\label{table:sc_performances}
\end{table}

\section{Conclusion and future work}
 A key result of our paper is the existence of a computationally efficient algorithm that is able to recover arbitrary discrete graphical models with multi-body interactions. This result is a particular case of the general framework that we have introduced, which considers arbitrary model parametrization and makes distinction between the bounds on the parameters of the underlying model and the prior parameters. The computational complexity $\widetilde{O}(p^{L})$ that we achieve is believed to be efficient for this problem \cite{Klivans2017}. In terms of sample complexity, the information-theoretic bounds for recovery of general discrete graphical models are unknown. In the case of binary pairwise models, the sample complexity bounds resulting from our general analysis are near-optimal with respect to known information-theoretic lower bounds \cite{SanthanamWainwright2012}. It would be interesting to see if the $1/\alpha^4$ factor in our sample complexity bounds can be improved to $1/\alpha^2$ using an $\ell_1$-norm penalty rather than an $\ell_1$-norm constraint, as it has been shown for the particular case of Ising models \cite{lokhov2018science,Vuffray2016nips}.

Other open questions left for future exploration include the possibility to extend the analysis to the case of graphical models with nonlinear parametrizations like in \cite{abhijith2020learning}, and to graphical models with continuous variables. It is particularly interesting to see whether the computationally efficient and nearly sample-optimal method introduced in the present work could be useful for designing efficient learning algorithms that can improve the state-of-the-art in the well-studied case of Gaussian graphical models, for which it has been recently shown that the information-theoretic lower bound on sample complexity is tight \cite{misra2018information}.

\newpage

\section*{Acknowledgments}
Research presented in this article was supported by the Laboratory Directed Research and Development program of Los Alamos National Laboratory under project numbers 20190059DR, 20190195ER, 20190351ER, and 20210078DR.

\vskip 0.2in
\bibliographystyle{plain}
\bibliography{learning_references}

\newpage



\appendix 

\section*{Appendices}
Appendix~\ref{app:well_poseness} contains a detailed discussion about Conditions \emph{(C1)} and \emph{(C2)}. In Appendix~\ref{app:proof_theorems}, the reader can find the proofs for the error bound on \textsc{Grise} estimates. Appendix~\ref{app:computational_complexity} contains the proofs relatives to the computational complexity of \textsc{Grise} and its efficient implementation. In Appendix~\ref{app:structure}, the reader can find the proofs in connection with the NPC constant and relative to structure and parameter estimation with \textsc{Suprise}. Finally, Section~\ref{sec:special_cases} contains the proofs for the applications of \textsc{Suprise} for structure and parameter estimation of iconic special cases. 

\section{About well-posedness and local learnability conditions}\label{app:well_poseness}

\subsection{About condition \emph{(C1)}}
To illustrate further why Condition \emph{(C1)} is required, we look at a case where the local constraint set is trivial, i.e. $\mathcal{Y}_i = \Real^{\mathbf{K}_i}$ and we consider a model that violates Condition \emph{(C1)}. This implies that there exists a sequence $\xb_n \in \mathcal{X}_{i}$ such that $\xb_n^{\top} \widetilde{I}(\theta^*) \xb_n/\|\xb_{\T_i}\|^2 < \rho_n  $ with $\rho_n \rightarrow 0$. In the limit, we can find a vector $\underline{x}$ such that $\underline{x}^{\top} \widetilde{I}(\theta^*) \underline{x} = 0$ and $\|\underline{x}_{\T_i}\| = 1$.
In other words, it implies that for this model there exists a vertex $i$ and a perturbation vector $\xb \in \mathcal{X}_i$ such that $\xb_{\T_i} \neq 0$ and for which $\E{\left(\sum_{k \in \K_i} x_k g_{ik} (\sigmab_k) \right )^2} = 0$. Since the probability distribution in Eq.~\eqref{eq:mu_def} is positive, it further implies that for all configurations $\sigmab$ we have the functional equality  $\sum_{k \in \K_i} x_k g_{ik}(\sigmab_k)=0$. This enables us to locally reparameterize the distribution: 
\begin{align}
    \exp \left( \sum_{k \in \K_i} \theta^{*}_k f_k(\sigmab_k) \right ) &= \exp \left(c(\xb, \sigmab_{\setminus i}) + \sum_{k \in \K_i} \left(\theta^{*}_k - x_k\right) f_k(\sigmab_k) \right ),\label{eq:cluster_f}
\end{align}
where $c(\xb,\sigmab_{ \setminus i}) = \sum_{k\in \K_i} x_k \phi_{ik}(\sigmab_{ k \setminus i})$ is a sum of locally centered functions that does not involve the variable $\sigma_i$. At this point, we should distinguish between the two cases when the basis functions $f_k$ are centered or not. When the basis functions are centered, i.e. $f_k = g_k$, the residual in Eq.~\eqref{eq:cluster_f} is identically zero, $c(\xb,\sigmab_{  \setminus i})=0$. Therefore, the probability distribution of the model in \eqref{eq:mu_def} can be reparameterized entirely with $\theta^{*}_k \rightarrow \theta^{*}_k - x_k$ for $k\in \K_i$. It implies, as $\xb_{\T_i} \neq 0$, that there exists two parameterization of the same models with different target parameters and the model selection problem as stated in Definition~\ref{def:model_selection_problem} is ill-posed. In the case when the basis functions are not centered, i.e. $f_k \neq g_k$, it may not be possible to reparameterized the whole distribution of the model. However, the conditional probability distribution $\p(\sigma_i \mid \sigmab_{ \setminus i})$ can be reparameterized as it is proportional to $\exp \left( \sum_{k \in \K_i} \theta^{*}_k f_k(\sigmab_k) \right )$ and $\exp \left( \sum_{k \in \K_i} \left(\theta^{*}_k - x_k\right) f_k(\sigmab_k) \right )$ thanks to Eq.~\eqref{eq:cluster_f}. Thereby, even if the model is uniquely parameterized with $\thetab_{\T_i}$, local methods based on independent neighborhood reconstructions using conditional distributions will fail at selecting a unique model as shown in the following example.
\begin{example}\label{ex:c1_condition} Consider two family of models over two binary variables $\sigma, s \in \{-1,1\}$, parameterized by $\thetab$ and $\underline{\eta}$,
\begin{align}
    \mu_{\thetab}(\sigma, s) \propto \exp\left(\theta_1 \sigma (s-1) + \theta_2 s (\sigma -1) + \theta_3 (\sigma + s)\right), \label{eq:model_fkup}
\end{align}
and
\begin{align}
    \mu_{\underline{\eta}}(\sigma, s) \propto \exp\left(\eta_1 \sigma s+ \eta_2 \sigma + \eta_3 s \right). \label{eq:model_korrect}
\end{align}
Both models are equivalent through the invertible mapping $\eta_1 = \theta_1 + \theta_2$, $\eta_2 = \theta_3 - \theta_1$ and $\eta_3 = \theta_3 - \theta_2$. However the model in Eq.~\eqref{eq:model_korrect} that has centered basis functions satisfies \emph{(C1)} from Condition~\ref{master_condition}, while the model in Eq.~\eqref{eq:model_fkup} that has non-centered basis functions does not. This implies that the parameters $\thetab$ cannot be recovered by looking independently at conditional distributions as they are degenerate in this basis,
\begin{align}
    \p(\sigma \mid s) &\propto \exp\left((\theta_1 +\theta_2) \sigma s + (\theta_3 - \theta_1) \sigma \right),\label{example:top_cond}\\
    \p(s \mid \sigma) &\propto \exp\left((\theta_1 +\theta_2) \sigma s + (\theta_3 - \theta_2) s \right)\label{example:bottom_cond}.
\end{align}
Indeed the change of parameters $\theta_1 \rightarrow \theta_1 + \epsilon$, $\theta_2 \rightarrow \theta_2 - \epsilon$ and $\theta_3 \rightarrow \theta_3 + \epsilon$ leaves the conditional distribution \eqref{example:top_cond} unchanged while the change of parameters $\theta_1 \rightarrow \theta_1 - \epsilon$, $\theta_2 \rightarrow \theta_2 + \epsilon$ and $\theta_3 \rightarrow \theta_3 - \epsilon$ leaves the conditional distribution \eqref{example:bottom_cond} unaffected. Note that there does not exist a change of parameters that leaves both \eqref{example:top_cond} and \eqref{example:bottom_cond} unchanged. This is in agreement with the fact that the model in Eq.~\eqref{eq:model_fkup} is uniquely parameterized and can be in principle recovered by looking \emph{jointly} at both conditional distributions. 
\end{example}

For the specific models that we considered in Section~\ref{sec:special_cases}, the basis functions are always centered, which implies that failure to satisfy \emph{(C1)} means that the model selection problem is ill-posed.

\subsection{About condition \emph{(C2)}}
The bound on the interaction strength in \emph{(C2)} translates directly into a uniform bound on the conditional probabilities of the models as shown in the following lemma.
\begin{lemma}[Lower-Bounded Conditional Probabilities]\label{lem:cond_prop_bound}
Consider a graphical model with bounded maximum interaction strength of $\gamma = \max_{i\in \V} | \max_{\sigmab}\sum_{k \in \K_i} \theta^{*}_k g_{ik}(\sigmab_k)|$. Then for any two disjoint subsets of vertices $A, B\subseteq \V$ the conditional probability of $\sigmab_{A}$ given $\sigmab_{B}$ is bounded from below,
\begin{align}
\p(\sigmab_{A} \mid \sigmab_{B}) \geq \prod_{i \in A} \frac{\exp(-2\gamma)}{|\A_i|},\label{eq:cond_prop_bound}
\end{align}
where $|\A_i|$ is the alphabet size of $\sigma_i$.
\end{lemma}
\begin{proof}[Proof of Lemma~\ref{lem:cond_prop_bound}: Lower-bound on conditional probabilities]
We start by bounding the conditional probability of one variable $\sigma_i$ given the rest $\sigmab_{\setminus{i}}$. This is given by the following expression,
\begin{align}
    \p(\sigma_i \mid \sigmab_{\setminus{i}}) &= \frac{\exp\left(\sum_{k\in \K_i} \theta^*_k f_k(\sigmab_k)\right)}{\sum_{\sigma_i \in \A_i} \exp\left(\sum_{k\in \K_i} \theta^*_k f_k(\sigmab_k)\right)}, \\
    &= \frac{\exp\left(\sum_{k\in \K_i} \theta^*_k \left(g_{ik}(\sigmab_k) + \phi_{ik}(\sigma_{k\setminus i})\right)\right)}{\sum_{\sigma_i \in \A_i} \exp\left(\sum_{k\in \K_i} \theta^*_k \left(g_{ik}(\sigmab_k) + \phi_{ik}(\sigma_{k\setminus i})\right)\right)}, \\
    &=\frac{\exp\left(\sum_{k\in \K_i} \theta^*_k g_{ik}(\sigmab_k)\right)}{\sum_{\sigma_i \in \A_i} \exp\left(\sum_{k\in \K_i} \theta^*_k g_{ik}(\sigmab_k) \right)},
\end{align}
as the centering functions $\phi_{ik}(\sigma_{k\setminus i})$ are independent of $\sigma_i$. The last expression can be simply bounded away from zero,
\begin{align}
    \p(\sigma_i \mid \sigmab_{\setminus{i}}) = \frac{\exp\left(\sum_{k\in \K_i} \theta^*_k g_{ik}(\sigmab_k)\right)}{\sum_{\sigma_i \in \A_i} \exp\left(\sum_{k\in \K_i} \theta^*_k g_{ik}(\sigmab_k)\right)} \geq \frac{\exp(-2\gamma)}{|\A_i|},\label{eq:var_comparison_first}
\end{align}
using $\gamma \geq |\sum_{k\in \K_i} \theta^*_k g_{ik}(\sigmab_k)|$.

Now we consider the conditional probability of one variable $\sigma_i$ given a subset of variable $\sigmab_{B}$ where $B \subseteq\V$ and $i\notin B$. Denote the complementary set of ${i}$ and $B$ by $S=\V \setminus{\left(\{i\}\cup B\right)}$. Then using the chain rule and the inequality from Eq.~\eqref{eq:var_comparison_first} we find the following lower-bound,
\begin{align}
    \p(\sigma_i \mid \sigmab_{B}) &= \sum_{\sigmab_S} \p(\sigma_{i}, \sigmab_S \mid \sigmab_{B})\\
    &= \sum_{\sigmab_S} \p(\sigma_{i}\mid \sigmab_{S}, \sigmab_B )\p(\sigmab_S  \mid \sigmab_{B})\\
    &= \sum_{\sigmab_S} \p(\sigma_{i}\mid \sigmab_{\setminus{i}} )\p(\sigmab_S  \mid \sigmab_{B})\\
    &\geq \frac{\exp(-2\gamma)}{|\A_i|}.\label{eq:var_comparison_second}
\end{align}

Finally we consider the conditional probability of a set of variable $\sigmab_{A}$ given another set $\sigmab_{B}$ where $A \cap B = \emptyset$. Denote the vertices in $A$ by $\left\{1,2,\ldots,|A|\right\}$. Then using the chain rule and the inequality from Eq.~\eqref{eq:var_comparison_first}, we obtain the desired result,
\begin{align}
\p(\sigmab_{A} \mid \sigmab_{B}) &= \prod_{j=1,\ldots,|A|}\p(\sigma_j, \mid \sigmab_{B}, \sigma_{j+1},\ldots, \sigma_{|A|}),\\
&\geq \prod_{j=1,\ldots,|A|} \frac{\exp(-2\gamma)}{|\A_j|}.
\end{align}
\end{proof}

\section{Proofs of \textsc{Grise} estimation error bound}\label{app:proof_theorems}

\begin{proposition}[Gradient Concentration for \textsc{Grise}] \label{prop:gradient_bound_grise}
For some node $u\in V$, let $n>\frac{2e^{2\gamma}}{\epsilon^2_1}\log(\frac{2\mathbf{K}_u}{\delta_1})$, then with probability at least $1-\delta_1$ the components of the gradient of the GISO are bounded from above as
\begin{align}
    \|\nabla\Sc_n(\thetab^*_u) \|_{\infty} < \epsilon_1.
\end{align}
\end{proposition}

Define the residual of the first order Taylor expansion as 
\begin{align}
    \delta \Sc_n(\Delta, \thetab_u^*) = \Sc_n(\thetab_u^* + \Delta) - \Sc_n(\thetab_u^*)  - \langle\nabla \Sc_n(\thetab_u^*),\Delta \rangle.
\end{align}
\begin{proposition}[Restricted Strong Convexity for \textsc{Grise}] \label{prop:rsc_grise}
For some node $u\in \V$, let $n > \frac{2}{\epsilon_2^2}\log\left(\frac{2\mathbf{K}_u^2}{\delta_2}\right)$ and assume that Condition~\ref{master_condition} holds for some norm $\|\cdot\|$. Then, with probability at least $1-\delta_2$ the error of the first order Taylor expansion of the GISO satisfies
\begin{align}
    \delta \Sc_n(\Delta, \thetab_u^*) \geq \exp(-\gamma)\frac{\rho_u \|\Delta_{\T_u}\|^2 - \epsilon_2 \|\Delta\|_1^2}{2+\|\Delta\|_1}.
\end{align}
for all $\Delta \in \mathcal{X}_u \subseteq \Real^{\mathbf{K}_u}$.
\end{proposition}

We first prove Theorem~\ref{thm:grise} before proving the propositions.
\begin{proof}[Proof of Theorem~\ref{thm:grise}: Error Bound on \textsc{Grise}]
For some node $u\in\V$, let $n\geq \frac{2^{14} \gammahat^2 (1+\gammahat)^2 e^{4\gamma}}{\alpha^{4} \rho_u^{2}} \log(\frac{4\mathbf{K}_u^2}{\delta})$. As the estimate $\widehat{\thetab}_u$ is an $\epsilon$-optimal point of the GISO and $\thetab^{*}_u$ lies in the constraint set from Eq.~\eqref{eq:local_constraint}, we find that for $\Delta = \widehat{\thetab}_u - \thetab^*_u$
\begin{align}
 \epsilon &\geq \Sc_n(\widehat{\thetab}_u) - \Sc_n(\thetab^*_u)\\
 &= \langle\nabla \Sc_n(\thetab_u^*),\Delta \rangle +  \delta \Sc_n(\Delta, \thetab_u^*)\\
 &\geq - \|\nabla\Sc_n(\thetab^*_u) \|_{\infty} \|\Delta\|_{1} + \delta \Sc_n(\Delta, \thetab_u^*).
\end{align}
Using the union bound on Proposition~\ref{prop:gradient_bound_grise} and Proposition~\ref{prop:rsc_grise} with $\delta_1=\delta_2=\frac{\delta}{2}$ and
\begin{align}
\epsilon \leq \frac{\rho_u \alpha^{2} e^{-\gamma}}{20(1+\gamma)}, \epsilon_1 = \frac{\rho_u \alpha^{2} e^{-\gamma}}{40 \gammahat (1+\gammahat)}, \epsilon_2 = \frac{\rho_u \alpha^{2}}{80 \gammahat^2},
\end{align}
we can express the inequality as
\begin{align}
 \epsilon \geq - \epsilon_1 \|\Delta\|_{1} + e^{-\gamma}\frac{\rho_u \|\Delta_{\T_u}\|^2 - \epsilon_2 \|\Delta\|^{2}_{1}}{2 +\|\Delta\|_{1}}.
\end{align}
Since by assumptions $\|\thetab^*_u\|_1 \leq \gamma$ and $\|\widehat{\thetab}_u\|_1 \leq \gammahat$ for $\gamma\leq\gammahat$ as the estimate is an $\epsilon$-optimal point of the \emph{$\ell_1$-constrained} GISO, the error $\|\Delta\|_1$ is bounded by $2\gammahat$. By choosing 

and after some algebra, we obtain that
\begin{align}
\|\Delta_{\T_u}\| \leq \frac{\alpha}{2}.
\end{align}
\end{proof}

\subsection{Gradient concentration}
The components of the gradient of the GISO is given by 
\begin{align}
    \frac{\partial}{\partial \theta_k} \Sc_n(\thetab^{*}_u) = \frac{1}{n} \sum_{t=1}^{n} -g_{uk}(\sigmab_k^{(t)}) \exp \left( -\sum_{l \in \K_u} \theta^{*}_l g_{ul}(\sigmab_l^{(t)})   \right).
\end{align}
Each term in the summation above is distributed as the random variable 
\begin{align}
    X_{uk} = -g_{uk}(\sigmab_k) \exp \left( -\sum_{l \in \K_u} \theta^{*}_k g_{ul}(\sigmab_l)   \right) \quad \forall k \in \K_u \label{eq:gradient}.
\end{align}

\begin{lemma} \label{lem:zero_expectation}
For any $u \in \V$ and $k \in \K_u$, we have
\begin{align}
    \E{X_{uk}} = 0.
\end{align}
\end{lemma}
\begin{proof}
Simple computation.
\end{proof}

\begin{proof}[Proof of Proposition~\ref{prop:gradient_bound_grise}: Gradient Concentration for \textsc{Grise}]
The random variable $X_{uk}$ is bounded as
\begin{align}
    |X_{uk}| &= |g_{uk}(\sigmab_k)| \exp \left( -\sum_{k \in \K_u} \theta^*_k g_{uk}(\sigmab_k)  \right) \leq \exp(\gamma).
\end{align}
Using Lemma~\ref{lem:zero_expectation} and the Hoeffding inequality, we get
\begin{align}
    \p \left( \left|\frac{\partial}{\partial \theta_k} \Sc_n(\thetab^*_u)\right| > \epsilon_1 \right) < 2 \exp \left(- \frac{n \epsilon_1^2}{2 e^{2\gamma}} \right).  \label{eq:single_vertex_hoeffding}
\end{align}
The proof follows by using \eqref{eq:single_vertex_hoeffding} and the union bound over all $k \in \K_u$.
\end{proof}

\subsection{Restricted strong convexity}


We make use of the following deterministic functional inequality derived in \cite{Vuffray2016nips}.
\begin{lemma} \label{lem:functional_ineuqality}
The following inequality holds for all $z \in \Real$.
\begin{align}
    e^{-z} - 1 + z \geq \frac{z^2}{2 + |z|}.
\end{align}
\end{lemma}
\begin{proof}[Proof of Lemma~\ref{lem:functional_ineuqality}]
Note that the inequality is true for $z=0$ and the first derivative of the difference is positive for $z>0$ and negative for $z<0$.
\end{proof}

Let $H_{k_1k_2}$ denote the correlation between $g_{k_1}$ and $g_{k_2}$ defined as
\begin{align}
    H_{k_1k_2} = \E{g_{uk_1}(\sigmab_{k_1})g_{uk_2}(\sigmab_{k_2})},
\end{align}
and let $H = [H_{k_1k_2}] \in \Real^{|\K_u|\times|\K_u|}$ be the corresponding matrix. We define $\hat H$ similarly based on the empirical estimates of the correlation $\hat{H}_{k_1k_2} = \frac{1}{n} \sum_{t=1}^{n} g_{uk_1}(\sigmab_{k_1}^{(t)})g_{uk_2}(\sigmab_{k_2}^{(t)})$. The following lemma bounds the deviation between the above two quantities.

\begin{lemma} \label{lem:H_concentration}
Choose some node $u\in \V$. With probability at least $1 - 2\mathbf{K}_u^2\exp \left(-\frac{n \epsilon_2^2}{2}\right)$, we have
\begin{align}
    |\hat{H}_{k_1k_2} - H_{k_1k_2}| < \epsilon_2,
\end{align}
for all $k_1, k_2 \in \K_u$. 
\end{lemma}

\begin{proof}[Proof of Lemma~\ref{lem:H_concentration}]
Fix $k_1, k_2 \in \K_u$. Then the random variable defined as $Y_{k_1k_2} = g_{uk_1}(\sigmab_{k_1})g_{uk_2}(\sigmab_{k_2})$ satisfies $|Y_{k_1k_2}| \leq 1$. Using the Hoeffding inequality we get
\begin{align}
    \p \left( |\hat{H}_{k_1k_2} - H_{k_1k_2}| > \epsilon_2 \right) < 2\exp \left(-\frac{n \epsilon_2^2}{2}\right).
\end{align}
The proof follows by using the union bound over $k_1, k_2 \in \K_u$.
\end{proof}

\begin{lemma} \label{lem:rsc_interim}
The residual of the first order Taylor expansion of the \textsc{GISO} satisfies
\begin{align}
    \delta \Sc_n(\Delta, \thetab_u^*) \geq \exp(-\gamma)\frac{\Delta^T \hat{H} \Delta}{2 + \|\Delta\|_1}.
\end{align}
\end{lemma}
\begin{proof}[Proof of Lemma~\ref{lem:rsc_interim}]
Using Lemma~\ref{lem:functional_ineuqality} we have
\begin{align}
    \delta \Sc_n(\Delta, \thetab_u^*) &= \frac{1}{n}\sum_{t=1}^n \exp \left(-\sum_{k \in \K_u} \theta_k^* g_{uk}(\sigmab_k^{(t)})\right)\times \\
    & \qquad \left(\exp\left(-\sum_{k \in \K_u}\Delta_k g_{uk}(\sigmab_k^{(t)})\right) - 1 + \sum_{k \in \K_u}\Delta_k g_{uk}(\sigmab_k^{(t)})\right) \\
    & \geq \exp(-\gamma) \frac{\Delta^T \hat{H} \Delta}{2 + |\sum_{k \in \K_u}\Delta_k g_{uk}(\sigmab_k^{(t)})|}.
\end{align}
The proof follows by observing that $|\sum_{k \in \K_u}\Delta_k g_{uk}(\sigmab_k^{(t)})| \leq \|\Delta\|_1$.
\end{proof}

We are now in a position to complete the proof of Proposition~\ref{prop:rsc_grise}.

\begin{proof}[Proof of Proposition~\ref{prop:rsc_grise}: Restricted Strong Convexity for \textsc{Grise}]
Using Lemma~\ref{lem:rsc_interim} we have
\begin{align}
   \delta \Sc_n(\Delta, \thetab_u^*)  &\geq \exp(-\gamma)\frac{\Delta^T \hat{H} \Delta}{2+\|\Delta\|_1} \\
    &= \exp(-\gamma)\frac{\Delta^T {H} \Delta + \Delta^T (\hat{H} - H) \Delta}{2+\|\Delta\|_1} \\
    & \stackrel{(a)}{\geq} \exp(-\gamma)\frac{\Delta^T {H} \Delta - \epsilon_2 \|\Delta\|_1^2}{2+\|\Delta\|_1} \\
    & \stackrel{(b)}{\geq} \exp(-\gamma)\frac{ \rho_u \|\Delta_{\T_u}\|^2 - \epsilon_2 \|\Delta\|_1^2}{2+\|\Delta\|_1}.
\end{align}
where $(a)$ follows from Lemma~\ref{lem:H_concentration} and $(b)$ follows from Condition~\ref{master_condition} as
\begin{align}
    \Delta^T {H} \Delta = \E{ \left(\sum_{k \in K_u} \Delta_k g_{uk} (\sigmab_k) \right )^2}.
\end{align}
\end{proof}

\section{Efficient implementation of \textsc{Grise}  and its computational complexity}\label{app:computational_complexity}

\begin{algorithm}
\SetAlgoLined
\tcp{Step 1: Initialization}
$x^{1}_{k,+} \leftarrow 1 / (2\mathbf{K}_u + 1), x^{1}_{k,-} \leftarrow 1 / (2\mathbf{K}_u + 1)$, $\forall k\in \K_u$\; 
$y^{1} \leftarrow 1 / (2\mathbf{K}_u + 1)$, $\eta^1 \leftarrow \sqrt{2\ln{(2\mathbf{K}_u + 1)}}$\;

\tcp{Step 2: Entropic Descent Steps}
\For{$t = 1,\ldots, T$}{
    \tcp{Gradient Update:}
    $v_{k} =  \frac{\partial}{\partial \theta_k} \Sc(\gammahat (\xb^{t}_{+} - \xb^{t}_{-})) / \Sc(\gammahat (\xb^{t}_{+} - \xb^{t}_{-}))$\;
    $w^{+}_k = x^{t}_{k,+} \exp(- \eta^t v_k)$, $w^{-}_k = x^{t}_{k,-} \exp( \eta^t v_k).$\;

    \tcp{Projection Step:}
    $z = y^{t} + \sum_{k\in \mathcal{K}_u} (w^{+}_k + w^{-}_k)$\;

    $x^{t+1}_{k,+} \leftarrow \frac{w^{+}_k}{z}$, $x^{t+1}_{k,-} \leftarrow \frac{w^{-}_k}{z}$\;  
        
    $y^{t+1} \leftarrow \frac{y^{t}}{z}$\;

    \tcp{Step Size Update:} 
    $\eta^{t+1} \leftarrow \eta^{t} \sqrt{\frac{t}{t+1}}$\;
    
}
\tcp{Step 3:} 
$s=\argmin_{t=1,\dots,T}{\Sc(\gammahat(\xb^{t}_{+} - \xb^{t}_{-}))}$\;
{\bf return} $\widehat{\thetab}_u = \gammahat (\xb^{s}_{+} - \xb^{s}_{-})$\;
\caption{Entropic Descent for unconstrained \textsc{Grise}}
\label{alg:mgd}
\end{algorithm}
\vspace{0.1in}

The iterative Algorithm~\ref{alg:mgd} takes as input a number of steps $T$ and output an $\epsilon$-optimal solution of \textsc{Grise} without constraints in Eq.~\eqref{eq:grise_unc}. This algorithm is an application of the Entropic Descent Algorithm introduced by \cite{Teboulle2003} to reformulation of Eq.~\eqref{eq:GRISE} as a minimization over the probability simplex. Note that there exist other efficient iterative methods for minimizing the GISO, such as the mirror gradient descent of \cite{BMN2001}. The following proposition provides guarantees on the computational complexity of unconstrained \textsc{Grise}.
\begin{proposition}[Computational Complexity for Unconstrained \textsc{Grise}]\label{prop:computational_complexity}
Let $1\geq\epsilon>0$ be the optimality gap and $T\geq 6 \epsilon^{-2} \ln{(2\mathbf{K}_u + 1)}$ be the maximum number of iterations. Then Algorithm~\ref{alg:mgd} is guaranteed to produce an $\epsilon$-optimal solution of \textsc{Grise} without a constraint set $\mathcal{Y}_u$ with a number of operation less than $C \frac{c_g n \mathbf{K}_u}{\epsilon^2}  \ln (1 + \mathbf{K}_u)$, where $c_g$ is an upper bound on the computational complexity of evaluating any $g_{ik}(\sigmab_k)$ for $k \in \K_i$ and $C$ is a universal constant that is independent of all parameters of the problem.
\end{proposition}
\begin{proof}[Proof of Proposition~\ref{prop:computational_complexity}: Computational complexity of unconstrained \textsc{Grise}]
We start by showing that the minimization of \textsc{Grise} in Eq.~\eqref{eq:GRISE} in the unconstrained case where $\mathcal{Y}_u = \mathbb{R}^{\mathbf{K}_u}$ is equivalent to the following lifted minimization on the logarithm of \textsc{Grise},
\begin{align}
    \min_{\thetab_u, \xb^{+}, \xb^{-}, y} \quad & \log \Sc_n(\thetab_u)\label{const:obj} \\
                    \text{s.t.} \quad & \thetab_u = \gammahat (\xb^{+} - \xb^{-}) \label{const:equality} \\  
                    & y + \sum_{k} (x^{+}_{k} + x^{-}_{k}) = 1 \label{const:sum_simplex} \\ 
                    & y\geq 0, x^+_k \geq 0, x^-_k \geq 0, \forall k\in K_u. \label{const:positive_simplex}
\end{align}
We first show that for all $\thetab_u \in \mathbb{R}^{\mathbf{K}_u}$ such that $\| \thetab_u \|_1 \leq \gammahat$, there exists $\xb^+, \xb^-, y$ satisfying constraints \eqref{const:equality}, \eqref{const:sum_simplex}, \eqref{const:positive_simplex}. This is easily done by choosing $x^+_k = \max(\theta_k / \gammahat ,0)$ , $x^-_k = \max(-\theta_k/ \gammahat ,0)$ and $y= 1 - \|\thetab_u\|_1 / \gammahat$.
Second, we trivially see that for all $\thetab_u, \xb^+, \xb^-, y$ satisfying constraints \eqref{const:equality}, \eqref{const:sum_simplex}, \eqref{const:positive_simplex}, it implies that $\thetab_u$ also satisfies $\| \thetab_u \|_1 \leq \gammahat$. Therefore, any $\thetab^{\rm min}_u$ that is an argmin of Eq.~\eqref{const:obj} is also an argmin of Eq.~\eqref{eq:GRISE} without constraint set $\mathcal{Y}_u$. Moreover, if we find $\thetab^{\rm \epsilon}_u$ such that $\log \Sc_n(\thetab^{\rm \epsilon}_u) - \log \Sc_n(\thetab^{\min}_u)\leq \epsilon/\sqrt{3}$, we obtain an $\epsilon$-minimizer of Eq.~\eqref{eq:GRISE} without constraint set $\mathcal{Y}_u$. Indeed, since $\epsilon/\sqrt{3}\leq \log(1+\epsilon)$ for $1\geq\epsilon>0$, we have that $\Sc_n(\thetab^{\rm \epsilon}_u) - \Sc_n(\thetab^{\min }_u)\leq \Sc_n(\thetab^{\min \epsilon}_u) \epsilon \leq \epsilon$ as $\Sc_n(\thetab^{\min }_u)\leq \Sc_n(0)=1$.
The remainder of the proof is a straightforward application of the analysis of the Entropic Descent Algorithm in \cite[Th.~5.1]{Teboulle2003} to the above minimization where $\thetab_u$ has been replaced by $\xb^+, \xb^-, y$ using Eq.~\eqref{const:equality}. In this analysis we use the fact that the logarithm of \textsc{Grise} remains a convex function as it is a sum of exponential functions and also that the gradient of our objective function is bounded uniformly by $\|\nabla \log \Sc_n(\thetab_u)\|_{\infty} =  \|\nabla \Sc_n(\thetab_u) / \Sc_n(\thetab_u)\|_{\infty}\leq 1$ as $|g_k(\sigmab_k)|\leq 1$. Note that the computational complexity of the gradient evaluation is proportional to $n \mathbf{K}_u c_g$. This is because for each sample, one has to first compute an exponential containing $\mathbf{K}_u$ terms $g_k(\sigmab_k)$ with an evaluation cost of $c_g$ and then multiply the exponential by the factor $-g_k(\sigmab_k)$ corresponding to each of the $\mathbf{K}_u$ components of the gradient.
\end{proof}

When the constraint set $\mathcal{Y}_u$ is \specialset, an $\epsilon$-optimal solution to \eqref{eq:GRISE} can be found by first solving the unconstrained version of \textsc{Grise} and then performing an equi-cost projection onto $\mathcal{Y}_u$. We define an $\epsilon$-optimal solution to the unconstrained \textsc{Grise} problem as
\begin{align}
    \Sc_n(\widehat{\thetab}_u^{unc}) \leq \min_{\thetab_u:\|\thetab_{u}\|_{1} \leq \gammahat} \Sc_n(\thetab_u) + \epsilon. \label{eq:grise_unc}
\end{align}
\begin{lemma} \label{lem:two_steps} Let $\widehat{\thetab}_u^{unc}$ be an $\epsilon$-optimal solution of the unconstrained \textsc{Grise} problem in Eq.~\eqref{eq:grise_unc}. Then an equi-cost projection of $\widehat{\thetab}_u^{unc}$ is an $\epsilon$-optimal solution of the constrained \textsc{Grise} problem,
    \begin{align}
        \Sc_n(\mathcal{P}_{\mathcal{Y}_u}(\widehat\thetab_u^{\text{unc}})) \leq \min_{\thetab_u\in \mathcal{Y}_u:\|\thetab_{u}\|_{1} \leq \gammahat} \Sc_n(\thetab_u) + \epsilon.
    \end{align}
\end{lemma}
\begin{proof}[Proof of Lemma~\ref{lem:two_steps}]
Since unconstrained  \textsc{Grise} is a relaxation of \textsc{Grise} with the constraints $\thetab_u \in \mathcal{Y}_u$, we must have
\begin{align}
    \Sc_n(\widehat\thetab_u^{\text{unc}})  \leq \min_{\thetab_u:\|\thetab_{u}\|_{1}  \leq \gammahat} \Sc_n(\thetab_u)+ \epsilon\leq  \min_{\thetab_u\in \mathcal{Y}_u:\|\thetab_{u}\|_{1} \leq \gammahat} \Sc_n(\thetab_u) + \epsilon.
\end{align}
Since, $\mathcal{Y}_u$ is \specialset, by definition, 
\begin{align}
    \Sc_n\left(\mathcal{P}_{\mathcal{Y}_u}(\widehat\thetab_u^{\text{unc}})\right) = \Sc_n(\widehat\thetab_u^{\text{unc}}).    \label{eq:two_step_proof}
\end{align}
The estimates $\mathcal{P}_{\mathcal{Y}_u}(\widehat\thetab_u^{\text{unc}})$ are feasible for the constrained \textsc{Grise} problem, completing the proof.
\end{proof}
Lemma~\ref{lem:two_steps} implies that the computational complexity of \textsc{Grise} for \specialset\ cases is the sum of the computational complexity of the unconstrained \textsc{Grise} and the projection step.

\begin{algorithm}
\SetAlgoLined
\tcp{Step 1: Solve unconstrained \textsc{Grise}}
Use Algorithm~\ref{alg:mgd} to obtain solutions $\widehat\thetab_u^{\text{unc}}$ to the unconstrained \textsc{Grise} \; 

\tcp{Step 2: Perform projection step}
Project $\widehat\thetab_u^{\text{unc}}$ onto $\mathcal{Y}_u$ to obtain the final estimates\;
$\widehat\thetab_u = \mathcal{P}_{\mathcal{Y}_u}(\widehat\thetab_u^{\text{unc}})$\;

{\bf return} $\widehat{\thetab}_u$\;
\caption{Computing \textsc{Grise} estimates for \specialset\ constraints}
\label{alg:overall_algorithm}
\end{algorithm}
\vspace{0.1in}

\noindent Algorithm~\ref{alg:overall_algorithm} is an implementation of \textsc{Grise} for \specialset\ cases. Its computational complexity is obtained easily by combining Lemma~\ref{lem:two_steps} and Proposition~\ref{prop:computational_complexity}.

\begin{theorem}[Computational Complexity for \textsc{Grise} with \specialsetShort\ Constraints] \label{th:computational_complexity}
Let $\mathcal{Y}_u$ be a \specialset\ set and let $1\geq\epsilon>0$ be given. Then Algorithm~\ref{alg:overall_algorithm} computes an $\epsilon$-optimal solution to \textsc{Grise} with a number of operations bounded by $C \frac{c_g n \mathbf{K}_u }{\epsilon^2} \ln (1 + \mathbf{K}_u) + \mathcal{C}(\mathcal{P}_{\mathcal{Y}_u}(\widehat\thetab_u^{\text{unc}}))$, where $c_g$ is an upper bound on the computational complexity of evaluating any $g_{ik}(\sigmab_k)$ for $k \in \K_i$ and where $\mathcal{C}(\mathcal{P}_{\mathcal{Y}_u}(\widehat\thetab_u^{\text{unc}}))$ denotes the computational complexity of the projection step.
\end{theorem}

\section{Proofs \& algorithms for structure and parameter estimation}\label{app:structure}

\subsection{Dimension independence and easier computation of NPC constants}
We recall the definition of the NPC constant,
\begin{align}   
  \rho^{\rm{NPC}}_i = \min_{\substack{c \in \Maxcli{\graph} \\ c \ni i}} \min_{\substack{\|\xb_c\|_2 = 1 \\ \xb_c \in \mathcal{X}^c_i}} \Econd{\sigma_i}{\sum_{\sigmab_{c\setminus i}\in \A_{c\setminus i}}\left(\sum_{k\in \cspan{c}}x_k h_{k}(\sigmab_k)\right)^2}.\label{eq:app_npc_def_recall}
\end{align}
In order to give an intuition for the intricate formula in Eq.~\eqref{eq:app_npc_def_recall}, let us define for maximal cliques $c\in \Maxcli{\graph}$, their clique parameterization matrix $G^{c}$, with indices being maximal factors $k, k' \in \cspan{c}$. The clique parameterization matrix is obtained by summing over variable $\sigmab_c$ the globally centered basis functions, 
\begin{align}
  G^{c}_{k,k'} = \sum_{\sigmab_{c}\in \A_{c}} h_{k}(\sigmab_c)h_{k'}(\sigmab_c). \label{eq:clique_matrix}
\end{align}
Note that the clique parametertization matrices are positive semi-definite matrices by construction. Bounding the expectation over $\sigma_i$ using Lemma~\ref{lem:cond_prop_bound}, we see that the NPC constant is linked to the smallest eigenvalue of the clique parameterization matrix, 
\begin{align}   
  \rho^{\rm{NPC}}_i \geq \frac{\exp(-2\gamma)}{q_i} \lambda_{\rm{min}}(G^c). \label{eq:clique_to_npc}
\end{align}
Clique parametertization matrices have a typical size of $\mathcal{O}(q^L\times q^L)$ since variables in a clique can take up to $\mathcal{O}(q^L)$ different configurations. Therefore, Eq.~\eqref{eq:clique_to_npc} emphasizes that that the NPC constant does not depend on the dimension of the model $p$ but rather on local properties of the parameterization of the family.

\subsection{Proofs for local learnability condition from nonsingular parametrization of cliques}
\begin{proof}[Proof of Proposition~\ref{prop:L_wise_NPC}: LLC in $\ell_{\infty,2}$-norm]
For a given vertex $i \in \V$, let  $\xb \in \mathcal{X}_i \subseteq \Real^{\mathbf{K}_i}$ be a vector in the perturbation set. First, suppose that $\{i\}$ is not a maximal clique and choose any maximal clique $c\in \Maxcli{\graph}$ that contains the vertex $i$ and let the set $S= c \setminus \{i\}$ be the set of nodes in the clique without $i$. The expression characterizing the LLC can be evaluated conditioning the expectation over $S$. Denoting the marginal and the conditional probability distribution used to compute expectation by subscripts, we find,
\begin{align}
    \E{\left(\sum_{k \in \K_i} x_k g_{ik} (\sigmab_k) \right )^2} 
    &= \Econd{\sigmab_{\setminus S}}{\Econd{\sigmab_S \mid \sigmab_{\setminus S}}{\left(\sum_{k \in \K_i} x_k g_{ik} (\sigmab_k) \right )^2}}, \label{eq:NPC_proof_llc}\\
    &\geq  \Econd{\sigmab_{\setminus S}}{\prod_{j \in S}\frac{\exp(-2\gamma)}{\left|\A_j\right|} \sum_{\sigmab_S \in \A_S}\left(\sum_{k \in \K_i} x_k g_{ik} (\sigmab_k) \right )^2}, \label{eq:NPC_proof_first_step}
\end{align}
where in the last line, we bounded the probability of $\p(\sigmab_S \mid \sigmab_{\setminus S})$ using Lemma~\ref{lem:cond_prop_bound}. We want to rewrite the sum over $\sigmab_S$ in Eq.~\eqref{eq:NPC_proof_first_step} using globally centered functions $h_k$ for factors $k\in \cspan{c}$ instead of locally centered functions $g_{ik}$. Using definitions of locally centered functions in Eq.~\eqref{eq:center_basis} and globally centered functions in Eq.~\eqref{eq:globally_centered_factors}, we see that $ g_{ik}(\sigmab_k)= h_k\left(\sigmab_k\right) + R_{ik}(\sigmab_k)$, where 
\begin{align}
    R_{ik}(\sigmab_k) = - \sum_{\substack{r\in \powaset(\partial k)\setminus \emptyset\\ r\neq\{i\}}} \frac{(-1)^{|r|}}{|\A_r|}\sum_{\sigmab_r}f_{k}(\sigmab_k).\label{eq:NPC_proof_centering_garbagio}
\end{align}
The sum in Eq.~\eqref{eq:NPC_proof_first_step} can be expanded into the four contributions,
\begin{align}
\sum_{\sigmab_S \in \A_S}\left(\sum_{k \in \K_i} x_k g_{ik} (\sigmab_k) \right )^2 
&= \sum_{\sigmab_S \in \A_S}\left(\sum_{k \in \cspan{c}} x_k h_{k}(\sigmab_c) \right)^2\label{eq:NPC_proof_first_contribution} \\
     &+ \sum_{k \in \cspan{c}}\sum_{l \in \K_i \setminus{\cspan{c}}} x_k x_l \sum_{\sigmab_S \in \A_S} h_{k}(\sigmab_c)g_{il}(\sigmab_l) \label{eq:NPC_proof_sec_contribution}\\ 
     &+ \sum_{k \in \cspan{c}}\sum_{l \in \cspan{c}}  x_k x_l \sum_{\sigmab_S \in \A_S} h_{k}(\sigmab_c)  R_{il}(\sigmab_c)\label{eq:NPC_proof_third_contribution}\\
     &+\sum_{\sigmab_S \in \A_S}\left(\sum_{k \in \K_i \setminus{\cspan{c}}} x_k g_{ik}(\sigmab_k)+ \sum_{k \in \cspan{c}} x_k R_{ik}(\sigmab_c)\right)^2 \label{eq:NPC_proof_fourth_contribution}.
\end{align}
We start by evaluating the contribution from terms in Eq.~\eqref{eq:NPC_proof_sec_contribution}. For $k \in \cspan{c}$ and $l \in \K_i \setminus{\cspan{c}}$, there exists at least one node $u\in c$ such that $u \neq i$ and $u \notin \partial l$. Summing over the variable $\sigma_u$ cancels the expression,
\begin{align}
\sum_{\sigma_u \in \A_u} h_{k}(\sigmab_c) g_{il}(\sigmab_{l}) = 0,
\end{align}
as $h_{k}(\sigmab_c)$ is globally centered and $g_{il}(\sigmab_{l})$ does not depend on $\sigma_{u}$.

The contribution from Eq.~\eqref{eq:NPC_proof_third_contribution} is also null. To see this we expand the sums using the formula for the reminder in Eq.~\eqref{eq:NPC_proof_centering_garbagio},
\begin{align}
\sum_{\sigmab_S \in \A_S} h_{k}(\sigmab_c) R_{il}(\sigmab_{c}) = -\sum_{\substack{r\in \powaset(\partial l)\setminus \emptyset\\ r\neq\{i\}}} \frac{(-1)^{|r|}}{|\A_r|} \sum_{\sigmab_S \in \A_S} h_{k}(\sigmab_c) \sum_{\sigmab_r}f_{l}(\sigmab_c) = 0,
\end{align}
where the sum over $\sigmab_S$ vanishes as $h_{k}(\sigmab_c)$ depends on $\sigmab_r \neq \sigma_i$ while $\sum_{\sigmab_r}f_{l}(\sigmab_c)$ does not. 
As the contribution from Eq.\eqref{eq:NPC_proof_fourth_contribution} is non-negative, we can lower-bound Eq.~\eqref{eq:NPC_proof_first_step} by the following expression,
\begin{align}
    \E{\left(\sum_{k \in \K_i} x_k g_{ik} (\sigmab_k) \right )^2}
    &\geq  \prod_{j \in S}\frac{\exp(-2\gamma)}{\left|\A_j\right|} \Econd{\sigmab_{\setminus S}}{ \sum_{\sigmab_S \in \A_S}\left(\sum_{k \in \cspan{c}} x_k h_{k} (\sigmab_c) \right )^2},\\
    &\geq \prod_{j \in S}\frac{\exp(-2\gamma)}{\left|\A_j\right|} \rho^{\rm{NPC}}_i  \sum_{k \in \cspan{c}} x_k^2,  \label{eq:partial_llc_npc}
\end{align}
where in the last line we have recognized the definition of the NPC constant from Eq.~\eqref{eq:npc_const}.
Since \eqref{eq:partial_llc_npc} holds for any $c\in \Maxcli{\graph}$ that contains the vertex $i$, the Local Learnability Condition is satisfied for a weighted $\ell_{\infty,2}$-norm with LLC constant equal to $\rho^{\rm{NPC}}$,
\begin{align}
  \E{\left(\sum_{k \in \K_i} x_k g_{ik} (\sigmab_k) \right )^2} \geq \rho^{\rm{NPC}}_i \|\xb_{\T_i}\|_{w(\infty,2)}^2, 
\end{align}
where the weighted $\ell_{\infty,2}$-norm is defined as follows,
\begin{align}
    \|\xb_{\T_i}\|_{w(\infty,2)} &= \max_{\substack{ c \ni i\\ c \in \Maxcli{\graph} }} \sqrt{ \prod_{j \in c \setminus{\{i\}}}\frac{\exp(-2\gamma)}{\left|\A_j\right|} \sum_{k \in \cspan{c}} x^2_k}.\label{eq:NPC_proof_weighted_norm}
\end{align}
As the weighted $\ell_{\infty,2}$-norm in Eq.~\eqref{eq:NPC_proof_weighted_norm} is lower-bounded by the $\ell_{\infty,2}$-norm,
\begin{align}
   \|\xb_{\T_i}\|_{w(\infty,2)}^2 \geq \left(\frac{\exp(-2\gamma)}{q}\right)^{L-1}\|\xb_{\T_i}\|_{\infty,2}^2,
\end{align}
we have that the LLC is also satisfied for the $\ell_{\infty,2}$-norm with LLC constant equal to $\rho^{\rm{NPC}}_i \left(\frac{\exp(-2\gamma)}{q}\right)^{L-1}$.

When $\{i\}$ is a maximal clique, then $\K_i = \cspan{\{i\}}$ and it straightforward to see that
\begin{align}   
  \E{\left(\sum_{k \in \K_i} x_k g_{ik} (\sigmab_k) \right )^2} &= \Econd{\sigma_i}{\left(\sum_{k\in \cspan{\{i\}}}\xb_k h_{k}(\sigma_i)\right)^2}\\
    &\geq \rho^{\rm{NPC}}_i \sum_{k \in \cspan{\{i\}}} x^2_k.
\end{align}
\end{proof}

\begin{lemma}\label{lem:var_pairwise}
Let $\sigma \in \mathcal{A}$, be a discrete random variable with probability distribution $p(\sigma)$. Consider $x_\sigma \in \mathbb{R}$, a function defined over $\sigma$ that is centered, i.e. $\sum_{\sigma \in \mathcal{A}} x_{\sigma} = 0$. The variance of the function $x_{\sigma}$ is lower-bounded by,
\begin{align}
 \var { x_\sigma } \geq p_\text{min} \sum_{\sigma \in \mathcal{A}} x_{\sigma}^2,
\end{align}
where $p_\text{min} = \min_{\sigma \in \mathcal{A}} p(\sigma)$.
\end{lemma}
\begin{proof}
The proof goes as follows,
\begin{align}
\var { x_\sigma } &= \sum_{\sigma \in \mathcal{A}} p(\sigma) \left(x_\sigma - \sum_{\sigma' \in \mathcal{A}} p(\sigma') x_{\sigma'}\right)^2\\
&\geq p_\text{min} \sum_{\sigma \in \mathcal{A}} \left(x_\sigma - \sum_{\sigma' \in \mathcal{A}} p(\sigma') x_{\sigma'}\right)^2\\
&= p_\text{min} \sum_{\sigma \in \mathcal{A}} \left( x_\sigma^2 - 2 x_\sigma \sum_{\sigma' \in \mathcal{A}} p(\sigma') x_{\sigma'} + \left(\sum_{\sigma' \in \mathcal{A}} p(\sigma') x_{\sigma'}\right)^2\right),\\
&\geq p_\text{min} \sum_{\sigma \in \mathcal{A}} x_{\sigma}^2,
\end{align}
where in the last line we used that $\sum_{\sigma \in \mathcal{A}} x_{\sigma} = 0$ and $\left(\sum_{\sigma' \in \mathcal{A}} p(\sigma') x_{\sigma'}\right)^2\geq 0$.
\end{proof}

\begin{proof}[Proof of Proposition~\ref{prop:pairwise_NPC}: LLC in $\ell_{2}$-norm for pairwise models]
For a given vertex $i \in \V$, let  $\xb \in \mathcal{X}_i \subseteq \Real^{\mathbf{K}_i}$ be a vector in the perturbation set. When $\{i\}$ is a maximal clique, then $\K_i = \cspan{\{i\}}$ and we immediately see that
\begin{align}   
  \E{\left(\sum_{k \in \K_i} x_k g_{ik} (\sigmab_k) \right )^2} &= \Econd{\sigma_i}{\left(\sum_{k\in \cspan{\{i\}}}\xb_k h_{k}(\sigma_i)\right)^2}\\
    &\geq \rho^{\rm{NPC}}_i \sum_{k \in \cspan{\{i\}}} x^2_k.
\end{align}

Now suppose that $\{i\}$ is not a maximal clique, i.e. there exists $j\in \V$ such that $\{i,j\}\in \Maxcli{\graph^*}$. The expectation that arises in the LLC is lower-bounded by its variance,
\begin{align}
    \E{\left(\sum_{k \in \K_i} x_k g_{ik} (\sigmab_k) \right )^2} &\geq \var{\sum_{k \in \K_i} x_k g_{ik} (\sigmab_k)}. \label{eq:pairwise_NPC_exptovar}
\end{align}

 Let $\{S_r\}_{r=1,\ldots,\chromo}$ be a minimal coloring of the graph $\graph^*$. For a given color $r$, define the set $C_r = S_r \setminus \{i\}$ and apply the law of total variance on the right-hand side of Eq.~\eqref{eq:pairwise_NPC_exptovar}, conditioning on $\sigmab_{\setminus C_r}$,
\begin{align}
   \var{\sum_{k \in \K_i} x_k g_{ik} (\sigmab_k)}  &\geq \Econd{\sigmab_{\setminus C_r}}{\varcond{\sigmab_{C_r} \mid \sigmab_{\setminus C_r}}{ \sum_{k \in \K_i} x_k g_{ik} (\sigmab_k) }},\label{eq:pairwise_NPC_condvar_total}
\end{align}
where the marginal and the conditional probability distribution used to compute expectation and variance respectively are indicated by subscripts. As the variance on the right-hand side of Eq.~\eqref{eq:pairwise_NPC_condvar_total} is conditioned on $\sigmab_{\setminus C_r}$, only basis functions involving a pair $(\sigma_i,\sigma_j)$ with $j\in C_r$ are giving a non-zero contribution to the conditional variance,
\begin{align}
   \varcond{\sigmab_{C_r} \mid \sigmab_{\setminus C_r}}{ \sum_{k \in \K_i} x_k g_{ik} (\sigmab_k) }  &=  \varcond{\sigmab_{C_r} \mid \sigmab_{\setminus C_r}}{ \sum_{j \in C_r} \sum_{k \in \cspan{\{i,j\}}} x_k g_{ik} (\sigma_i,\sigma_j) }.\label{eq:pairwise_NPC_condvar_partial}
\end{align}
We can rewrite the locally centered functions with respect to globally centered functions using their definitions found in Eq.~\eqref{eq:center_basis} and in Eq.~\eqref{eq:globally_centered_factors},
\begin{align}
g_{ik}(\sigma_i,\sigma_j) = h_k(\sigma_i,\sigma_j) - \frac{1}{|\A_i|}\sum_{\sigma_j}f_k(\sigma_i,\sigma_j) + \frac{1}{|\A_i||\A_j|}\sum_{\sigma_i,\sigma_j}f_k(\sigma_i,\sigma_j).\label{eq:pairwise_NPC_local_global_centering}
\end{align}
We see from Eq.~\eqref{eq:pairwise_NPC_local_global_centering} that the difference between locally and globally centered functions only depends on the variable $\sigma_i$. This means that we can interchange locally centered functions with globally centered functions in the right-hand side of Eq.~\eqref{eq:pairwise_NPC_condvar_partial} as the variance is conditioned on $\sigma_i$,
\begin{align}
   \varcond{\sigmab_{C_r} \mid \sigmab_{\setminus C_r}}{ \sum_{j \in C_r} \sum_{k \in \cspan{\{i,j\}}} x_k g_{ik} (\sigma_i,\sigma_j) }  &=  \varcond{\sigmab_{C_r} \mid \sigmab_{\setminus C_r}}{ \sum_{j \in C_r} \sum_{k \in \cspan{\{i,j\}}} x_k h_{k} (\sigma_i,\sigma_j) }.\label{eq:pairwise_NPC_condvar_partial_global}
\end{align}

Since $\{S_r\}_{r=1,\ldots,\chromo}$ is a vertex coloring, by definition all nodes $j\in C_r$ having the same color are not sharing a factor node, i.e. $\forall j_1, j_2\in C_r$, $\nexists k\in \K^*$ such that $j_1, j_2\in \partial k$. This implies that variables $\sigma_{j}$ with $j\in C_r$ are independent conditioned on the remaining variables $\sigmab_{\setminus C_r}$ and the variance in Eq.~\eqref{eq:pairwise_NPC_condvar_partial_global} can be rewritten,
\begin{align}
   \varcond{\sigmab_{C_r} \mid \sigmab_{\setminus C_r}}{ \sum_{j \in C_r} \sum_{k \in \cspan{\{i,j\}}} x_k h_{k} (\sigma_i,\sigma_j) } &=  \sum_{j \in C_r} \varcond{\sigmab_{j} \mid \sigmab_{\setminus C_r}}{  \sum_{k \in \cspan{\{i,j\}}} x_k h_{k} (\sigma_i,\sigma_j) }.
   \label{eq:pairwise_NPC_condvar_serparate}
\end{align}
The right-hand side of Eq.~\eqref{eq:pairwise_NPC_condvar_serparate} is centered with respect to $\sigma_j$ and we can apply Lemma~\ref{lem:var_pairwise} and Lemma~\ref{lem:cond_prop_bound} to find a lower-bound that is only dependant on the random variable $\sigma_i$,
\begin{align}
  \sum_{j \in C_r} \varcond{\sigmab_{j} \mid \sigmab_{\setminus C_r}}{  \sum_{k \in \cspan{\{i,j\}}} x_k h_{k} (\sigma_i,\sigma_j) }
   &\geq \frac{\exp(-2\gamma)}{q} \sum_{j \in C_r} \sum_{\sigma_j \in \A_j} \left(\sum_{k \in \cspan{\{i,j\}}} x_k h_{k} (\sigma_i,\sigma_j) \right)^2.\label{eq:pairwise_NPC_condvar_lowerbound}
\end{align}
Plugging back the results derived in Eq.~\eqref{eq:pairwise_NPC_condvar_partial_global}, Eq.~\eqref{eq:pairwise_NPC_condvar_serparate}, and Eq.~\eqref{eq:pairwise_NPC_condvar_lowerbound} into the initial inequality in Eq.~\eqref{eq:pairwise_NPC_condvar_total}, we find,
\begin{align}
   \var{\sum_{k \in \K_i} x_k g_{ik} (\sigmab_k)}  &\geq \Econd{\sigma_i}{\frac{\exp(-2\gamma)}{q} \sum_{j \in C_r} \sum_{\sigma_j \in \A_j} \left(\sum_{k \in \cspan{\{i,j\}}} x_k h_{k} (\sigma_i,\sigma_j) \right)^2},\\
    &\geq  \frac{\exp(-2\gamma)}{q} \rho_{i}^{\rm{NPC}} \sum_{j \in C_r} \sum_{k \in \cspan{\{i,j\}}} x^2_k, \label{eq:pairwise_NPC_para_matrix}
\end{align}
where in Eq.~\eqref{eq:pairwise_NPC_para_matrix} we used the definition of the NPC constant in Eq.~\eqref{eq:npc_const} to bound the quadratic form involving $\xb$.

Finally, we average the inequality described by Eq.~\eqref{eq:pairwise_NPC_para_matrix} over the different colors and hence possible conditioning sets $C_r$ to conclude the proof,
\begin{align}
   \var{\sum_{k \in \K_i} x_k g_{ik} (\sigmab_k)} 
    & \geq \frac{\exp(-2\gamma)}{q} \frac{\rho^{\rm{NPC}}_i}{\chromo} \sum_{r=1,\ldots,\chromo} \sum_{j \in C_r} \sum_{k \in \cspan{\{i,j\}}} x^2_k,\\
    &= \frac{\exp(-2\gamma)}{q} \frac{\rho^{\rm{NPC}}_i}{\chromo} \sum_{\substack{k \in \MaxFac{\graph} \\ \partial k \ni {i}}} x^2_k.
\end{align}
\end{proof}

\subsection{Proofs of estimation guarantees for the \textsc{Suprise} algorithm}

\begin{proof}[Proof of Theorem~\ref{thm:structure_general}: Reconstruction and Estimation Guarantees for \textsc{Suprise}] As the NPC constant is non-zero $\rho^{\rm{NPC}}_u>0$ for all nodes $u \in \V$, we apply Proposition~\ref{prop:L_wise_NPC} in conjunction with Theorem~\ref{thm:grise} to find that for each step $t\in \{0,\ldots,L-1\}$ and with probability at least $1-\delta/(pL)$, \textsc{Grise} around a node $u\in \V$ recovers the parameters in each maximal clique $c \in \Maxcli{\graph\left[(\V,\K^t)\right]}$ that contains $u$ with precision $\sum_{k \in \cspan{c}} (\thetab_k - \widehat\thetab_k)^2 \leq (\alpha/2)^2$. 
Therefore, at each step $t\in \{0,\ldots,L-1\}$ and with probability at least $1-\delta/L$, the factor removal procedure is guaranteed to remove all factors of size $L-t$ that are not present in the graph if all factors of size bigger than $L-t$ were correctly removed in the previous steps. Since there are at most $L$ removal steps, it implies that the overall procedure discovers all maximal cliques with probability at least $1-\delta$.

\end{proof}

\section{Application to special cases}  \label{sec:special_cases}
In this section, we show how to apply Theorem~\ref{thm:structure_general} in order to derive the sample and computational complexity of reconstructing graphical models for some common basis functions.



\subsection{Binary models on the monomial basis}\label{sub:binary_monomial}
In this subsection, we consider general models on binary alphabet $\A_i=\{-1,1\}$. Let the factors be all nonempty subsets of $\{1,\dots,p\} = \V$ of size at most $L$, 
\begin{align}
    \K = \{k \subseteq \V \mid |k| \leq L\}.
\end{align}
The set $\K$ contains all potential subsets of variable of size at most $L$. The parameterization uses the monomial basis given by $f_k(\sigmab_k) =  \prod_{j\in k} \sigma_j $ with $k\in \K$. Note that the monomial basis functions are already globally centered $f_k \equiv g_k \equiv h_k$. The probability distribution for this model is expressed as
\begin{align}
    \mu_{\text{binary}}(\sigmab) = \frac{1}{Z} \exp \left( \sum_{k \in \K} \theta^{*}_{k} \prod_{j\in k} \sigma_j \right).\label{eq:binary_distribution}
\end{align}
When $L \leq 2$, the model in Eq.~\eqref{eq:binary_distribution} is pairwise and it is referred as the Ising Model.

For each maximal clique there exists exactly one maximal factors in its span. Therefore, the NPC constant as defined in Eq.~\eqref{eq:npc_const} is $\rho_{\rm{NPC}} = 1$ since for any clique $c$ we have,
\begin{align}
    \sum_{\sigmab_{c\setminus u}\in \A_{c\setminus u}}\left(\sum_{k\in \cspan{c}}x_k h_{k}(\sigmab_k)\right)^2 = x_k^2 \sum_{\sigmab \in \left\{-1,1\right\}^{|c|-1}}\left(\prod_{i=1}^{|c|} \sigma_i\right)^2 = 2^{|c|-1} x_k^2,
\end{align}
and the minimum is achieved for cliques of size one.
As every node is involved in at most $\mathbf{K} \leq p^{L-1}$ factor functions, the structure of binary models can be recovered as a corollary of Theorem~\ref{thm:structure_general}.
\begin{corollary}[\textbf{Structure recovery for binary graphical models}]\label{cor:binary_models}
Let $\sigmab^{(1)}, \ldots, \sigmab^{(n)}$ be i.i.d. samples drawn according to $\mu(\sigmab)$ in \eqref{eq:binary_distribution} and let $\alpha \leq \min_{k\in\struct(\graph^*)} |\theta^*_k|$ be the intensity of the smallest non-zero parameter. If
\begin{align}
   n\geq 2^{12} 4^L \frac{\gammahat^2 (1+\gammahat)^2 e^{4\gamma L}}{\alpha^{4}} \log\left(\frac{4 L p^{2L-1}}{\delta}\right),   \label{eq:n_multi}
\end{align}
then the structure of the binary graphical model is perfectly recovered using Algorithm~\ref{alg:suprise}, i.e. $\widehat{\mathbbm{S}} = \mathbbm{S}(\thetab^*)$,  with probability $1-\delta$. Moreover the total computational complexity scales as $\widetilde{\mathcal{O}}(p^L)$, for fixed $L$, $\alpha$, $\gamma$, $\gammahat$ and $\delta$.
\end{corollary}

For pairwise Ising models that are $\chromo$ colorable, we have also guarantees on the $\ell_2$-norm reconstruction by \textsc{Suprise} of pairwise parameters.

\begin{corollary}[$\ell_2$-parameter estimation for Ising models]\label{cor:ising_l2}
Let $\sigmab^{(1)}, \ldots, \sigmab^{(n)}$ be i.i.d. samples drawn according to $\mu(\sigmab)$ in \eqref{eq:binary_distribution} for $L=2$ and let $\alpha>0$ be the prescribed estimation accuracy. If
\begin{align}
    n\geq 2^{16}\frac{ \gammahat^2 (1+\gammahat)^2  \chromo^2 e^{8 \gamma}}{\alpha^{4}} \log\left(\frac{8p^3}{\delta}\right), 
\end{align}
then, with probability at least $1-\delta$, the parameters are estimated by Algorithm~\ref{alg:suprise} with the error
\begin{align}
    \sqrt{\sum_{i,j\in V} |\widehat{\thetab}_{ij} - \thetab^{*}_{ij} |^2} \leq  \frac{\alpha}{2}. \label{eq:l2_estimates_ising}
\end{align}
The computational complexity of obtaining these estimates is $\widetilde{\mathcal{O}}(p^2)$ for fixed $\chromo$, $\alpha$, $\gamma$, $\gammahat$ and $\delta$. 
\end{corollary}
As graphs with bounded degree $d$ have a chromatic number at most $d+1\geq \chromo$, Corollary~\ref{cor:ising_l2} recovers the $\ell_2$-guarantees for sparse graphs recovery of \cite{Vuffray2016nips} albeit with slightly worse dependence with respect to $\gamma$ and $\alpha$. The worse $\gamma$ dependence is an artifact of the general analysis presented in this paper. For models over binary variables one can improve the $e^{8\gamma}$ dependence to $e^{6 \gamma}$ using Berstein's inequality in Proposition~\ref{prop:gradient_bound_grise} instead of Hoeffding's inequality. However, the worse $\alpha$ dependence seems to be more fundamental. It is caused by the replacement of the $\ell_1$-penalty used in \cite{Vuffray2016nips} by an $\ell_1$-constraint.

For graphs with unbounded vertex degree but low chromatic number, such as star graphs or bipartite graphs, Corollary~\ref{cor:ising_l2} shows that the parameters of the corresponding Ising model can be fully recovered with a bounded $\ell_2$-error using a number of samples that is only logarithmic in the model size $p$.


\subsection{L-wise models with arbitrary alphabets on the indicator basis}\label{sub:arbitrary}
In this subsection, we consider $L$-wise graphical models over variables taking values in arbitrary alphabet $\mathcal{A}_i$ of size $q_i$, parametrized with indicator-type functions. The building block of the set of basis functions is the centered univariate indicator function defined as
\begin{align} 
    \Phi_{s_i,\sigma_i} = 
    \begin{cases}
      1-\frac{1}{q_i}, & \text{if } \sigma_i = s_i, \\ \label{def:indicator_function}
      -\frac{1}{q_i}, & \text{otherwise,}
    \end{cases}
\end{align}
where $s_i, \sigma_i \in \mathcal{A}_i$ are prescribed letters of the alphabet. The univariate indicator functions in Eq.~\eqref{def:indicator_function} are centered Kronecker delta functions and possess similar properties such as symmetry $\Phi_{s_i,\sigma_i} = \Phi_{\sigma_i,s_i}$ and contraction under a summation,
\begin{align}
\sum_{\tau_i \in \A_i} \Phi_{\tau_i, s_i} \Phi_{\tau_i,\sigma_i} = \Phi_{s_i,\sigma_i}. \label{eq:contraction_phi}
\end{align}
The set of factors $\K$ are pairs associating elements of $R = \{r \in \powaset(\V) \mid |r|\leq L \}$ which are subsets of variable of size at most $L$ with an alphabet configuration in $\mathcal{A}_r = \bigotimes_{i\in r} \mathcal{A}_i$, 
\begin{align}
    \K = \{(r,\underline{s}_r) \mid r \in R, \underline{s}_r \in \A_r\}.
\end{align}
In what follows, we slightly abuse the notation of factors and parameters by shortening $(r,\underline{s}_r) \equiv \underline{s}_r$. With these notations, the indicator basis functions are constructed as $f_{\underline{s}_r}(\sigmab_r) = \prod_{i \in r}\Phi_{s_i,\sigma_i}$. Note that the indicator basis functions are globally centered i.e. $f_{\underline{s}_r}\equiv g_{\underline{s}_r}\equiv h_{\underline{s}_r}$.
The probability distribution of an $L$-wise graphical model with arbitrary alphabet is defined as follows,
\begin{align}
    \mu_{\text{general}}(\sigmab) = \frac{1}{Z} \exp \left( \sum_{r\in R} \sum_{\underline{s}_r \in  \mathcal{A}_r} \theta^{*}_{\underline{s}_r} \prod_{i\in r} \Phi_{s_i,\sigma_i}\right).\label{eq:arbitrary_model}
\end{align}
The family of distribution in Eq.~\eqref{eq:arbitrary_model} is not uniquely parameterized by the parameters $\thetab^{*}$. To see this, we introduce the linear application $\mathcal{P}_r$ acting on arrays  $\theta_{\underline{s}_r}$ as follows,
\begin{align}
    [\mathcal{P}_r\thetab]_{\sigmab_r} = \sum_{\underline{s}_r \in \A_r}\theta_{\underline{s}_r} \prod_{i\in r} \Phi_{s_i,\sigma_i}.\label{eq:projector_phi}
\end{align}
Using the contraction property from Eq.~\eqref{eq:contraction_phi}, it is easy to see that $\mathcal{P}_r$ is a projector, i.e $\mathcal{P}^2_r = \mathcal{P}_r$. It is also straightforward to verify that $\mathcal{P}_r\thetab$ is always a globally centered array and if $\thetab$ is already globally centered then $\mathcal{P}_r\thetab = \thetab$. Therefore, the applications $\mathcal{P}_r$ are projectors on the space of array $\theta_{\underline{s}_r}$ which are globally centered, i.e. $\sum_{s_i}\theta_{\underline{s}_r} = 0$ for all $i\in r$.
We lift the parametrization degeneracy in Eq.~\eqref{eq:arbitrary_model} by imposing that parameters $\thetab^*$ are in the range of the projector $\mathcal{P}_r$. We thus require that the parameters satisfy the following linear constraints at each vertex $u\in\V$,
\begin{align}
    \mathcal{Y}_u = &\bigcap_{\substack{r \in R \\ r \ni u}}\left\{ \thetab_u \in \mathbbm{R}^{\mathbf{K}_u} \;\middle|\; \forall i \in r, \sum_{s_i \in \mathcal{A}_i} \theta_{\underline{s}_r} = 0 \right\}.\label{eq:arbitrary_constraint_set}
\end{align}
The constraint set in Eq.~\eqref{eq:arbitrary_constraint_set} is \specialset\ according to Definition~\ref{def:special_constraint_set} as we explicitly exhibited the equi-cost projection $\{\mathcal{P}_r\}_{r\in R}$ onto it. The computational complexity of this projection is no more than $\mathcal{O}(p^{L-1}q^L)$.

As the constraint set in Eq.~\eqref{eq:arbitrary_constraint_set} forms a linear subspace, the perturbation set is simply $\mathcal{X}_u = \mathcal{Y}_u \cap B_1(2\gammahat)$, the intersection of the constraint set with  the $\ell_1$-ball of radius $2\gammahat$. Maximal cliques are subset of vertices and hence are also elements of $R$. Therefore, the NPC constant as defined in Eq.~\eqref{eq:npc_const} is bounded by $\rho_{\rm{NPC}} \geq \exp(-2\gamma)/q$ since for each clique we have,
\begin{align}
    \Econd{\sigma_i}{\sum_{\sigmab_{c\setminus i}\in \A_{c\setminus i}}\left(\sum_{k\in \cspan{c}}x_k h_{k}(\sigmab_k)\right)^2} &\geq \frac{\exp(-2\gamma)}{q_i} \sum_{\sigmab_{c}\in \A_{c}} \left( \left[\mathcal{P}_c \xb \right]_{\sigmab_c} \right)^2,\\
    &=\frac{\exp(-2\gamma)}{q_i} \sum_{\sigmab_{c}\in \A_{c}} x^2_{\sigmab_c},
\end{align}
as $\xb \in \mathcal{X}_u$ is globally centered and thus is in the range of the projector $\mathcal{P}_c$.
Every node is involved in at most $\mathbf{K} \leq p^{L-1}q^L$ factor functions and the structure of L-wise models with arbitrary alphabets can be recovered as a corollary of Theorem~\ref{thm:structure_general}.

\begin{corollary}[\textbf{Structure recovery for L-wise graphical models}]\label{cor:arbitrary_models}
Let $\sigmab^{(1)}, \ldots, \sigmab^{(n)}$ be i.i.d. samples drawn according to $\mu(\sigmab)$ in \eqref{eq:arbitrary_model} and let $\alpha \leq \min_{c\in\struct(\graph^*)} \sqrt{\sum_{\underline{s}_c\in \A_c}{\theta^*_{\underline{s}_c}}^2}$ be the intensity of the smallest non-zero parameter. If
\begin{align}
   n\geq 2^{14} q^{2L} \frac{\gammahat^2 (1+\gammahat)^2 e^{4\gamma (L+1)}}{\alpha^{4}} \log\left(\frac{4 L q^{2L}p^{2L-1}}{\delta}\right),   \label{eq:struct_arbitrary_models}
\end{align}
then the structure of the L-wise graphical model with arbitrary alphabets is perfectly recovered using Algorithm~\ref{alg:suprise}, i.e. $\widehat{\mathbbm{S}} = \mathbbm{S}(\thetab^*)$,  with probability $1-\delta$. Moreover the total computational complexity scales as $\widetilde{\mathcal{O}}(p^L)$, for fixed $L$, $q$, $\alpha$, $\gamma$, $\gammahat$ and $\delta$.
\end{corollary}

For pairwise models with arbitrary alphabet that are $\chromo$ colorable, we have also guarantees on the $\ell_2$-norm reconstruction by \textsc{Suprise} of pairwise parameters.

\begin{corollary}[$\ell_2$-parameter estimation for pairwise models]\label{cor:pairwise_l2}
Let $\sigmab^{(1)}, \ldots, \sigmab^{(n)}$ be i.i.d. samples drawn according to $\mu(\sigmab)$ in \eqref{eq:arbitrary_model} for $L=2$ and let $\alpha>0$ be the prescribed estimation accuracy. If
\begin{align}
   n\geq 2^{14} q^4 \frac{\gammahat^2 (1+\gammahat)^2 \chromo^2 e^{12\gamma}}{\alpha^{4}} \log\left(\frac{8 q^{4}p^{3}}{\delta}\right), 
\end{align}
then, with probability at least $1-\delta$, the parameters are estimated by Algorithm~\ref{alg:suprise} with the error
\begin{align}
    \sqrt{\sum_{i,j \in V} \sum_{s_i \in \A_i,  s_j \in \A_j} |\widehat{\thetab}_{s_i,s_j} - \thetab^{*}_{s_i,s_j} |^2} \leq  \frac{\alpha}{2}. \label{eq:l2_estimates_pairwise}
\end{align}
The computational complexity of obtaining these estimates is $\widetilde{\mathcal{O}}(p^2)$ for fixed $\chromo$, $q$, $\alpha$, $\gamma$, $\gammahat$ and $\delta$. 
\end{corollary}

\end{document}